\newif\ifarxiv\arxivtrue

\documentclass[10pt]{article}

\usepackage{booktabs}
\usepackage{makecell}

\usepackage[pagebackref]{hyperref}

\ifarxiv
\usepackage{fullpage}
\usepackage{authblk}
\usepackage{xcolor}
\usepackage{natbib}
\else 
\usepackage[accepted]{icml2024}
\fi

\usepackage{algorithm}
\usepackage{algorithmic}
\usepackage{times}

\definecolor{Gred}{RGB}{219, 50, 54}
\definecolor{Ggreen}{RGB}{60, 186, 84}
\definecolor{Gblue}{RGB}{72, 133, 237}
\definecolor{Gyellow}{RGB}{247, 178, 16}
\definecolor{ToCgreen}{RGB}{0, 128, 0}
\definecolor{myGold}{RGB}{231,141,20}
\definecolor{myBlue}{rgb}{0.19,0.41,.65}
\definecolor{myPurple}{RGB}{175,0,124}
\definecolor{niceRed}{RGB}{153,0,0}
\definecolor{niceRed}{RGB}{190,38,38}
\definecolor{blueGrotto}{HTML}{059DC0}
\definecolor{royalBlue}{HTML}{057DCD}
\definecolor{navyBlueP}{HTML}{0B579C}
\definecolor{limeGreen}{HTML}{81B622}
\definecolor{nicePink}{RGB}{247,83,148}
\definecolor{linkDarkBlue}{rgb}{0,0.08,0.45}

\usepackage{cmap}
\usepackage[T1]{fontenc}
\usepackage{bm}

\usepackage{amsmath}
\usepackage{amsfonts}
\usepackage{amssymb}
\usepackage{amsbsy}
\usepackage{amsthm}
\usepackage{thm-restate}
\usepackage{mathtools}
\usepackage{bbm}
\usepackage{physics}

\usepackage{rotating}
\usepackage{float}
\usepackage{tikz}

\usepackage[shortlabels]{enumitem}

\hypersetup{
	colorlinks = true,
	urlcolor = {linkDarkBlue},
	linkcolor = {linkDarkBlue},
	citecolor = {linkDarkBlue}
}

\usepackage{chngcntr}

\counterwithin{equation}{section}
\counterwithin{table}{section}
\usepackage{nicefrac}
\usepackage[nameinlink,noabbrev]{cleveref}

\definecolor{myC}{rgb}{0, 255, 255}
\definecolor{myY}{rgb}{204, 204, 0}
\definecolor{myM}{rgb}{255, 0, 255}
\definecolor{secinhead}{RGB}{249,196,95}
\definecolor{lgray}{gray}{0.8}

\newtheorem{theorem}{Theorem}[section]
\newtheorem*{theorem*}{Theorem} 
\newtheorem*{proposition*}{Proposition} 
\newtheorem{lemma}[theorem]{Lemma}

\newtheorem{proposition}[theorem]{Proposition}

\newtheorem{corollary}[theorem]{Corollary}

\newtheorem{definition}[theorem]{Definition}

\renewcommand{\Pr}{\mathop{\bf Pr\/}}
\newcommand{\E}{\mathop{\bf E\/}}

\newcommand{\R}{\mathbb{R}}

\usepackage{pgfplots}

\newcommand{\poly}{\textnormal{poly}}

\newcommand{\sgn}{\textnormal{sgn}}

\newcommand{\reals}{\mathbb R}

\newcommand{\nats}{\mathbb N}

\newcommand{\cD}{\mathcal{D}}

\DeclareMathOperator{\supp}{supp}

\newcommand{\eps}{\epsilon}

\newcommand{\calA}{\mathcal{A}}

\newcommand{\calD}{\mathcal{D}}

\newcommand{\calH}{\mathcal{H}}

\newcommand{\calN}{\mathcal{N}}

\newcommand{\calR}{\mathcal{R}}

\newcommand{\calX}{\mathcal{X}}
\newcommand{\calY}{\mathcal{Y}}

\def\<{\langle}
\def\>{\rangle}

\newcommand{\ones}{\mathbbm{1}}

\def\wt{\widetilde}

\newcommand{\AKround}{\mathtt{AKround}}

\newcommand{\SGD}{\mathtt{SGD}}
\newcommand{\boostSGD}{\mathtt{boostSGD}}
\newcommand{\SVM}{\mathtt{SVM}}
\newcommand{\rLearnerFinite}{\mathtt{rLearnerFinite}}

\newcommand{\sset}{\subseteq}

\DeclarePairedDelimiter{\set}{\{}{\}}

\let\abs\relax
\DeclarePairedDelimiter{\abs}{|}{|}
\let\norm\relax
\DeclarePairedDelimiter{\norm}{\lVert}{\rVert}

\begin{document}
\date{}

\ifarxiv
\title{Replicable Learning of Large-Margin Halfspaces\thanks{Authors listed alphabetically.}}
\onecolumn
\author[1]{Alkis~Kalavasis\thanks{\texttt{alvertos.kalavasis@yale.edu}}}
\author[1,2]{Amin~Karbasi\thanks{\texttt{amin.karbasi@yale.edu}}}
\author[3]{Kasper~Green~Larsen\thanks{\texttt{larsen@cs.au.d}}}
\author[1]{\authorcr Grigoris~Velegkas\thanks{\texttt{grigoris.velegkas@yale.edu}}}
\author[1]{Felix~Zhou\thanks{\texttt{felix.zhou@yale.edu}}}

\affil[1]{Yale University}
\affil[2]{Google Research}
\affil[3]{Aarhus University}

\maketitle

\else

\icmltitlerunning{Replicable Learning of Large-Margin Halfspaces}

\twocolumn[
\icmltitle{
    Replicable Learning
    of Large-Margin Halfspaces
}

\icmlsetsymbol{equal}{*}

\begin{icmlauthorlist}
    \icmlauthor{Alkis Kalavasis}{equal,yale}
    \icmlauthor{Amin Karbasi}{equal,yale,google}
    \icmlauthor{Kasper Green Larsen}{equal,aarhus}
    \icmlauthor{Grigoris Velegkas}{equal,yale}
    \icmlauthor{Felix Zhou}{equal,yale}
\end{icmlauthorlist}

\icmlaffiliation{yale}{Yale University}
\icmlaffiliation{google}{Google Research}
\icmlaffiliation{aarhus}{Aarhus University}

\icmlcorrespondingauthor{Alkis Kalavasis}{alvertos.kalavasis@yale.edu}
\icmlcorrespondingauthor{Grigoris Velegkas}{grigoris.velegkas@yale.edu}
\icmlcorrespondingauthor{Felix Zhou}{felix.zhou@yale.edu}

\icmlkeywords{
Learning Theory, Large-Margin Halfspaces, Replicability, Reproducibility
}

\vskip 0.3in
]

\printAffiliationsAndNotice{\icmlEqualContribution}
\fi

\begin{abstract}
We provide efficient replicable algorithms for the problem of learning large-margin halfspaces.
Our results improve upon the algorithms provided by \citet*[STOC,][]{impagliazzo2022reproducibility}. 
We design the first dimension-independent replicable algorithms for
this task which runs in polynomial time, 
is proper, 
and has strictly improved sample complexity compared to the one
achieved by \citet{impagliazzo2022reproducibility} 
with respect to all the relevant parameters.
Moreover,
our first algorithm has sample complexity that is optimal with respect to the accuracy parameter
$\eps$.
We also design an SGD-based replicable algorithm that, in some parameters' regimes, achieves better sample and time complexity than our first algorithm.  
Departing from the requirement of polynomial time algorithms, using the DP-to-Replicability reduction of \citet*[STOC,][]{bun2023stability}, 
we show how to obtain a replicable algorithm for large-margin halfspaces with improved sample complexity with respect
to the margin parameter $\tau$,
but running time doubly exponential in $1/\tau^2$ and worse sample
complexity dependence on $\eps$ than one of our previous algorithms.
We then design an improved algorithm with better sample complexity
than all three of our previous algorithms
and running time exponential in $1/\tau^{2}$.
\end{abstract} 

\section{Introduction}
The replicability crisis is omnipresent in many scientific disciplines including
biology, medicine, chemistry, and, importantly, AI \citep{baker20161,pineau2019iclr}. A recent article
that appeared in Nature \citep{ball2023ai}
explains how the reproducibility crisis in AI
has a cascading effect across many other scientific areas
due to its widespread applications in other fields such as medicine.
Thus, 
an urgent goal is to design
a formal framework through which we can argue about the replicability of experiments in ML. Such a theoretical framework was proposed in a recent work by \citet{impagliazzo2022reproducibility} and has been studied extensively in several learning settings
\citep{esfandiari2023replicable,esfandiari2023replicableb,bun2023stability,kalavasis2023statistical,chase2023replicability,dixon2023list,chase2023local,eaton2023replicable,karbasi2023replicability}.

\begin{definition}
[Replicability \citep{impagliazzo2022reproducibility}]
\label{def:replicable}
Let $\calR$ be a distribution over random binary strings. A learning algorithm
$\calA$ is $n$-sample $\rho$-replicable if for any distribution $\calD$ over inputs and two independent datasets $S, S' \sim \calD^n$,
it holds that $\Pr_{S,S' \sim \calD^n, r \sim \calR} [\calA(S,r) \neq \calA(S',r)] \leq \rho$.
\end{definition}
In words, this definition requires that when an algorithm $\calA$ is executed
twice on different i.i.d.\ datasets $S, S'$ but using shared \emph{internal} randomness, 
then the output of the algorithm is \emph{exactly} the same, 
with high probability.
We note that sharing the randomness across the executions
is crucial in achieving this replicability guarantee.
Importantly, \citet{dixon2023list}
showed that without sharing randomness,
it is impossible to achieve such
a strong notion of replicability even for simple tasks such 
as mean estimation.
In practice, this can be achieved by simply
publishing the random seed that the ML algorithms are executed with. As we extensively discuss in \Cref{sec:priorwork}, \Cref{def:replicable} turns out to be connected with other notions of stability such as differential privacy and perfect generalization \citep{ghazi2021user,bun2023stability}.

In this work, we study the fundamental problem of learning large-margin halfspaces, which means that no example is allowed to lie too close to the separating hyperplane. This  task
is related to foundational ML techniques such as the Perceptron algorithm \citep{rosenblatt1958perceptron}, SVMs \citep{cortes1995support}, and AdaBoost \citep{freund1997decision}.
Let us recall the concept class of interest.
\begin{definition}
[Large-Margin Halfspaces]
\label{def:margin}
Let $\calD$ be a distribution over $\reals^d \times \{-1,1\}$ whose support does not contain $x=0$. We say that $\calD$ has linear margin $\tau$ if there exists a unit vector
$w \in \reals^d$ such that for any $(x,y) \in  \mathrm{supp}(\calD)$ it holds that
$y (w^\top x / \|x\|)  \geq \tau$.\footnote{When we do not specify the norm, we assume the $\ell_2$-norm.}
\end{definition}

Following the PAC learning definition of \citet{valiant1984theory}, 
we say that an algorithm learns with accuracy $\eps$ and confidence $\delta$ the class of $\tau$-margin halfspaces in $d$ dimensions using $n = n(\eps, \delta, d, \tau)$ samples and runtime $T =T(\eps,\delta,d,\tau)$ if, given $n$ i.i.d.\ samples from any distribution $\calD$ satisfying \Cref{def:margin}, the algorithm outputs, in time $T$, a classifier $h : \reals^d \to \{-1,1\}$ such that $\Pr_{(x,y) \sim \calD}[h(x) \neq y] \leq \eps$, with probability at least $1-\delta$.

We are interested in \emph{replicably} learning large-margin halfspaces, i.e., designing algorithms for large-margin halfspaces that further satisfy \Cref{def:replicable}.
We remark that when the feature domain is infinite,
there is no replicable learning algorithm
for learning halfspaces in general.
Thus making some assumptions like the large-margin condition is \emph{necessary}. 
In particular, 
\citet{bun2023stability, kalavasis2023statistical} show that finiteness of 
the Littlestone dimension is a necessary condition for learnability by
replicable algorithms, and it is known that even one-dimensional halfspaces over $[0,1]$
have infinite Littlestone dimension.
See \Cref{tab:summary work} for a comparison of prior work and our contributions.

\newcommand{\bigtableCaption}{
\caption{A comparison of prior work and our work.
    We denote by $d$ the dimension, $\eps$ the accuracy, $\rho$ the replicability, and $\tau$ the margin of the halfspace.
    We omit the logarithmic factors on the sample complexity and the runtime.}
}

\newcommand{\bigtable}{
\begin{tabular}{ p{5.5cm} p{3.3cm} p{5.5cm} p{1cm} }
    \toprule
    \multicolumn{4}{c}{\textbf{Replicable Algorithms for Large-Margin Halfspaces}} \\
    \midrule
    
    \hspace{19mm}\small{Algorithms}
    & \hspace{1.8mm}\small{Sample Complexity} 
    & \hspace{10mm}\small{Running Time} 
    & \small{Proper} \\
    \midrule
    
    \textbf{Prior Work} 
    & & & \\
    \midrule 
    \color{blue}{[ILPS22]} \color{black}with foams rounding & $(d \eps^{-3} \tau^{-8} \rho^{-2} )^{1.01} $   & $2^d \cdot \poly(1/\eps, 1/\rho, 1/\tau)$ & No\\
    \midrule
    \color{blue}{[ILPS22]} \color{black}with box rounding & $(d^3 \eps^{-4} \tau^{-10} \rho^{-2} )^{1.01} $   & $\poly(d,1/\eps, 1/\rho, 1/\tau)$ & No\\
    
    \midrule
    
    \textbf{Our Work} 
    & & &\\
    
    \midrule

    \Cref{alg:algo2} (\Cref{thm:main2}) & $\eps^{-1} \tau^{-7} \rho^{-2}$ & $\poly(d,1/\eps,1/\rho, 1/\tau)$ & Yes\\
    \midrule

    \Cref{alg:algo4} (\Cref{thm:main4}) & {$\eps^{-2} \tau^{-6} \rho^{-2}$} & {$\poly(d,1/\eps,1/\rho, 1/\tau)$} & Yes\\
    \midrule
    
    \color{blue}{[LNUZ20]}
    \color{black}
    via DP-to-Replicability reduction 
    \color{blue}{[BGH+23]} 
    \color{black} 
    (\Cref{prop:dp-reduction-sample-complexity})
   & $\eps^{-2} \tau^{-4} \rho^{-2} $   & $\poly(d) \cdot \exp\left( (\nicefrac1\tau)^{\frac{\log( 1/(\eps \rho))}{\tau^2}} \right) $ & Yes\\
    
    \midrule
    
    \Cref{alg:algo3} (\Cref{thm:main3}) & $\eps^{-1} \tau^{-4} \rho^{-2}$ & $\poly(d) \cdot  \poly(1/\eps,1/\rho, 1/\tau)^{1/\tau^2}$ & Yes\\
    \bottomrule
\end{tabular}
}

\ifarxiv

\begin{table}[!ht]
    \centering
    \bigtableCaption
    \label{tab:summary work}
    \bigtable
\end{table}

\else

\begin{table}[t]
\caption{Replicable Algorithms for Large-Margin Halfspaces.
We denote by $d$ the dimension, $\eps$ the accuracy, $\rho$ the replicability, and $\tau$ the margin of the halfspace.
We omit the logarithmic factors on the sample complexity.
See \Cref{apx:detailed work summary} for a more detailed comparison.
}
\label{tab:summary work}
\vskip 0.15in
\begin{center}
\begin{small}
\begin{sc}
\begin{tabular}{lcccr}
    \toprule
    Algorithm & Samples & Proper? \\
    \midrule
    
    \textbf{Prior Work} & & & \\
    
    \midrule 
    
    \makecell{\color{blue}{[ILPS22]} \\ (inefficient)}
    & $(d \eps^{-3} \tau^{-8} \rho^{-2} )^{1.01} $   
    & No\\
    
    \midrule
    
    \makecell{\color{blue}{[ILPS22]}}
    & $(d^3 \eps^{-4} \tau^{-10} \rho^{-2} )^{1.01} $   
    & No\\
    
    \midrule
    
    \textbf{Our Work} & & &\\
    
    \midrule

    \makecell{\Cref{alg:algo2}\\\Cref{thm:main2}} 
    & $\eps^{-1} \tau^{-7} \rho^{-2}$ 
    & Yes\\
    \midrule

    \makecell{\Cref{alg:algo4}\\\Cref{thm:main4}} 
    & {$\eps^{-2} \tau^{-6} \rho^{-2}$}
    & {Yes}\\
    \midrule
    
    \makecell{DP reduction\\ \small\color{blue}{[LNUZ20; BGH+23]}\\ \Cref{prop:dp-reduction-sample-complexity}\\ (inefficient)}
    & $\eps^{-2} \tau^{-4} \rho^{-2} $ 
    & Yes\\
    
    \midrule
    
    \makecell{\Cref{alg:algo3}\\\Cref{thm:main3}\\(inefficient)} %
    & $\eps^{-1} \tau^{-4} \rho^{-2}$ 
    & Yes\\

    \bottomrule
\end{tabular}
\end{sc}
\end{small}
\end{center}
\vskip -0.1in
\end{table}

\fi

The work of \citet{impagliazzo2022reproducibility} provided the first replicable algorithms for $\tau$-margin halfspaces over $\reals^d$. 
The first algorithm of \citet{impagliazzo2022reproducibility}, which uses the ``foams'' discretization scheme \citep{kindler2012spherical}, is $\rho$-replicable and returns a hypothesis $h$ that, with probability at least $1-\rho$,
satisfies $\Pr_{(x,y) \sim \calD}[h(x) \neq y] \leq \eps$. The
sample complexity of this algorithm is roughly $\wt{O}( (d \eps^{-3} \tau^{-8} \rho^{-2} )^{1.01})$ and the runtime is exponential in $d$ and polynomial in $1/\eps, 1/\rho$ and $1/\tau$.
The second algorithm of \citet{impagliazzo2022reproducibility}, which uses the ``box'' discretization scheme, is $\rho$-replicable and returns a hypothesis $h$ that, with probability at least $1-\rho$,
satisfies $\Pr_{(x,y) \sim \calD}[h(x) \neq y] \leq \eps$ 
with
sample complexity $\wt{O}( ( d^{3} \eps^{-4} \tau^{-10} \rho^{-2} )^{1.01})$ and runtime which is polynomial in $d, 1/\eps, 1/\rho$ and $1/\tau$.
These two algorithms appear in the first two rows of \Cref{tab:summary work}.

Some remarks are in order.
In our setting,
the sample complexity of learning large-margin
halfspaces in the absence of the replicability requirement is
$\widetilde O(\nicefrac{1}{\eps\tau^2}).$
Notice that the sample complexity of both algorithms of 
\citet{impagliazzo2022reproducibility}
depends on the dimension $d$ of the problem. 
This is unexpected since the sample complexity of non-replicable
algorithms for this task is dimension-independent.
In the case of the replicable algorithms of \citet{impagliazzo2022reproducibility},
the dependence on the 
dimension appears due to a rounding/discretization step, which is crucial 
in establishing the replicability guarantees. 
Second, both algorithms of \citet{impagliazzo2022reproducibility} are
\emph{improper} in the sense that the hypothesis $h$ they output does not correspond to a halfspace. 
This is due to the use of a replicable boosting routine that outputs the majority vote over multiple halfspaces. 
As a general note, both of these algorithms are fairly complicated:
they require multiple discretization/rounding steps, and then they output
a \emph{weak} learner, which finally needs to be boosted using
multiple rounds of a replicable boosting scheme. As a result,
the sample complexity of their algorithms incurs a significant blow-up
in the parameters $\eps,\tau$ compared to the non-replicable setting.

{We work in the setting of finite (but not necessarily bounded) bit precision with the goal of designing algorithms that are agnostic to the marginal distribution and sample complexity that is dimension-independent.
Indeed,
assuming bounded bit precision $b$,
possibly some structure about the marginal distribution,
and that the margin $\tau > 0$, 
there are replicable algorithms for learning halfspaces with time and sample complexity $O(\poly(d, b, 1/\epsilon, \log(1/\tau), 1/\rho))$ via black-box transformations of existing (complicated) SQ algorithms \citep{balcan2015statistical, dunagan2004simple}.
In the regime when $\tau$ is constant and $d$ is relatively large,
our algorithms can outperform the black-box transformations.}

\subsection{Our Contribution}
\citet{impagliazzo2022reproducibility} leave as an open question whether the sample complexity bounds of their algorithms are tight. 
In this work, we show that these bounds are sub-optimal. 
We provide new replicable algorithms for learning large-margin halfspaces that improve upon the results of \citet{impagliazzo2022reproducibility} in various aspects. 
Our algorithms have no sample dependence on $d$, 
strictly improve on the dependence on $1/\eps,1/\rho$, and $1/\tau$,\footnote{By iteratively halving a guess for $\tau$, we may assume without loss of generality that $\tau$ is known.}
and are proper,
meaning that they output linear models.
Moreover,
our \Cref{alg:algo2} and \Cref{alg:algo4} are computationally efficient
while \Cref{alg:algo3} forsakes computational efficiency to achieve further improvements in sample complexity.

We now state our first algorithmic result,
the proof of which can be found in \Cref{sec:algo2}.
\begin{theorem}
[Efficient Replicable \Cref{alg:algo2}]
\label{thm:main2}
    \sloppy
    {Fix $\eps, \tau, \rho, \delta\in (0, 1)$}.
    Let $\calD$ be a distribution over $\reals^d \times \{-1,1\}$ that has linear margin $\tau$ as in \Cref{def:margin}.
    There is an algorithm that is $\rho$-replicable and, given $m = \wt{O}(\eps^{-1} \tau^{-7} \rho^{-2} \log(\nicefrac1\delta))$ i.i.d. samples $(x,y) \sim \calD$, computes in time $\poly(d,1/\eps,1/\tau,1/\rho,\log(\nicefrac1\delta))$ a unit vector $w \in \reals^d$ such that
    $\Pr_{(x,y) \sim \calD}[\sgn(w^\top x) \neq y] \leq \eps$ with probability at least $1-\delta$.
\end{theorem}

\Cref{alg:algo2} improves on the sample complexity of the two algorithms appearing in \citet{impagliazzo2022reproducibility},  runs in polynomial time, and is proper.
Our techniques follow a different path from that of \citet{impagliazzo2022reproducibility}. As we alluded to before, 
their approach is fairly complicated and is based on the design of a replicable weak halfspace learning algorithm and then a replicable boosting algorithm that combines multiple weak learners. 
Our approach is single-shot and \emph{significantly simpler}: 
Consider $B$ independent non-overlapping batches of training examples.
From each batch $i \in [B]$, 
we find a hyperplane with normal vector $w_i \in \reals^d$
that has $\Omega(\tau)$ margin on the training
data.
This can be achieved by running the standard Support Vector Machine (SVM) algorithm \citep{cortes1995support, vapnik2006estimation}.
We then aggregate our vectors to a single average normal vector $z = (1/B) \sum_{i \in [B]} w_i$. 
Finally, we project the vector $z$ onto a lower-dimensional space,
whose dimension does not depend on $d,$ and
we replicably round $z$ using a rounding scheme due to \citet{alon2017optimal}
for which we perform a novel analysis in the shared randomness setting.
We emphasize that our algorithm gives a halfspace with the desired
accuracy guarantee without the need to use any boosting schemes.

At a technical level, we avoid the dependence on the dimension $d$ thanks to data-oblivious dimensionality reduction techniques (cf. \Cref{proof:JL}), a standard tool in the literature of large-margin halfspaces. Instead of rounding our estimates in the $d$-dimensional
space,
we first project them to a lower-dimensional space, whose 
dimension does not depend on $d$, and we perform the rounding in that
space. 
The crucial idea is that one can use the data-obliviousness of Johnson-Lindenstrauss (JL) matrices so that the projection matrices in two distinct executions are the same since the internal randomness is shared.
Another technical aspect of our algorithm
that differentiates it from prior works on the design of replicable 
algorithms is the use of a different rounding scheme known as the \emph{Alon-Kartag rounding scheme} (cf. \Cref{section:rounding}).
{While this rounding scheme was known,
to the best of our knowledge,
we are the first to utilize and analyze the scheme in the context of replicability.}
Consider the simple case of 1-dimensional data.
In the same spirit as in \citet{impagliazzo2022reproducibility}, 
we consider a random grid.
But rather than rounding the point
to a fixed element of each cell of the grid (e.g., its center),
we randomly round it to one of the two endpoints of the cell using shared internal randomness
so that, in expectation, 
the rounded point is the same as the original one. 
This is helpful as it preserves inner products in expectation,
and therefore gives better concentration properties 
across multiple roundings. 
We believe that this rounding scheme can find more applications
in the replicable learning literature.
The detailed proof of \Cref{thm:main2} can be found in \Cref{sec:algo2}.

{Despite the simplicity of our algorithm, there
are technical subtleties that complicate its analysis.
For instance, the projection to the low-dimensional space
introduces a subtle complication
we need to handle.
In particular, using ideas from \citet{gronlund2020near} we 
can show that the aggregated vector in the \emph{high-dimensional}
space has the desired generalization properties. However,
when we project it to the low-dimensional space there 
are vectors that are now misclassified, due to the error
introduced by the JL mapping. Using the guarantees of the JL
projection, we show that, uniformly over
the data-generating distributions, this happens for only a small
fraction of the population.}

{\Cref{alg:algo4} follows a similar framework as \Cref{alg:algo2} by running a non-replicable algorithm on independent batches of data
and then aggregating the outputs replicably,
illustrating the flexibility of our approach.
We now describe this result in more detail,
whose proof can be found in \Cref{sec:algo4}.}

\begin{theorem}
[Efficient Replicable \Cref{alg:algo4}]
\label{thm:main4}
    \sloppy
    {Fix $\eps, \tau, \rho, \delta\in (0, 1)$}.
    Let $\calD$ be a distribution over $\reals^d \times \{-1,1\}$ that has linear margin $\tau$ as in \Cref{def:margin}.
    There is an algorithm that is $\rho$-replicable and, given $m = \wt{O}(\eps^{-2} \tau^{-6} \rho^{-2} \log(\nicefrac1\delta))$ i.i.d. samples $(x,y) \sim \calD$, computes in time $\poly(d,1/\eps,1/\tau,1/\rho,\log(\nicefrac1\delta))$ a unit vector $w \in \reals^d$ such that
    $\Pr_{(x,y) \sim \calD}[\sgn(w^\top x) \neq y] \leq \eps$ with probability at least $1-\delta$.
\end{theorem}

Compared to \Cref{alg:algo2}, our \Cref{alg:algo4} achieves better dependence on $\tau$ by incurring an additional $1/\eps$ factor in the sample complexity. At a technical level, as in \cite{le2020efficient}, we provide a convex surrogate that upper bounds the loss $\ones \{y(x^\top w) \leq \tau/2\}$. Running SGD on this convex surrogate provides a unit vector $w$ that, in expectation over the data, achieves a margin of at least $\tau/2$ for an $O(\eps)$-mass of the population. We then apply a standard boosting trick to turn this guarantee into a high probability bound. Next, we work as in \Cref{alg:algo2}: we run the above procedure $B$ times to get $w_1,...,w_B$ and aggregate our vectors into a single vector $z = (1/B)\sum_{i \in [B]} w_i$. Lastly, we perform a JL-projection on $z$ and then round using the Alon-Klartag rounding scheme, as in \Cref{alg:algo2}.

\paragraph{Computationally Inefficient Reductions from DP.}
It is a corollary of the works of \citet{bun2023stability} and \citet{kalavasis2023statistical} that one can use existing differentially private (DP) algorithms in order to obtain replicable learners.
In particular, following the reduction of \citet{bun2023stability}, one can obtain a replicable algorithm for large-margin halfspaces with better sample complexity in terms of $\tau$, but in a \emph{computationally inefficient} way.
The idea is to take an off-the-shelf DP algorithm (recall \Cref{def:dp}) for this problem (e.g. \citet{le2020efficient})
and transform it into a replicable one. We remark that this
transformation holds when the algorithm outputs \emph{finitely} many different solutions and
its running time is exponential in the number
of these solutions. 
Fortunately, the pure DP algorithm from \citet{le2020efficient}
satisfies this finite co-domain property.
The formal statement of the result we get
by combining these two algorithms is presented below.

\begin{proposition}
[Inefficient Replicable Algorithm; follows from \citet{le2020efficient,bun2023stability}]\label{prop:dp-reduction-sample-complexity}
    {Fix $\eps, \tau, \rho, \delta\in (0, 1)$}.
    Let $\calD$ be a distribution over $\reals^d \times \{-1,1\}$ that has linear margin $\tau$ as in \Cref{def:margin}.
    There is an algorithm that is $\rho$-replicable and, given $m = \wt{O}(\eps^{-2} \tau^{-4} \rho^{-2} \log(\nicefrac1\delta))$ i.i.d. samples $(x,y) \sim \calD$, computes, in time~${\exp\left( (\nicefrac{1}{\tau})^{\frac{\log(1/(\epsilon \rho \delta))}{\tau^2}} \right)} %
    \cdot 
    \poly(d)
    $,~a unit vector $w \in \reals^d$ such that
    $\Pr_{(x,y) \sim \calD}[\sgn(w^\top x) \neq y] \leq \eps$, with probability at least $1-
    \delta$.
\end{proposition}

As mentioned above,
the proof of this result follows by combining the DP-to-Replicability transformation of \citet{bun2023stability} (cf. \Cref{prop:reduction})
with the {pure} DP algorithm for learning large-margin halfspaces
due to \citet{le2020efficient} (cf. \Cref{prop:lydia}).
We note that since the algorithm of \citet{le2020efficient} is proper and the reduction of \citet{bun2023stability} is based on sub-sampling, the output of \Cref{prop:dp-reduction-sample-complexity} is also a linear classifier.
The main issue with this approach is that, apart from not being a polynomial time algorithm, the reduction requires a \emph{quadratic blow-up} in the sample complexity of the provided DP algorithm. 

{To be more specific,
the DP algorithm of \cite{le2020efficient}
has sample complexity $\wt O(\eps^{-1} \tau^{-2})$ for accuracy $\eps$ and margin $\tau$. This means that the replicable algorithm of \Cref{prop:dp-reduction-sample-complexity} incurs a {quadratic
blow-up} in the sample complexity on the parameters $\eps,\tau$. The cost of this transformation is 
tight under standard cryptographic hardness assumptions \citep{bun2023stability}.
Thus, it is unlikely that we can reduce the dependence on the error parameter $\eps$ using such a generic transformation. We emphasize that our efficient
replicable algorithm (cf. \Cref{alg:algo2}) has linear sample
complexity
dependence on $1/\eps.$
The blow-up on the running time of the algorithm
is due to the use of \emph{correlated sampling} in the transformation of 
\citet{bun2023stability}
which requires exponential running time in the size
of the output space. We remark that in the case of the algorithm
of \citet{le2020efficient}, the size of the output space
is already exponential in $1/\tau^2.$
}

In our work, we also revisit this inefficient algorithm and improve on its sample complexity and runtime, as follows:
\begin{theorem}
[Improved Inefficient Replicable \Cref{alg:algo3}]
\label{thm:main3}
    Fix $\eps, \tau, \rho, \delta\in (0, 1)$.
    Let $\calD$ be a distribution over $\reals^d \times \{-1,1\}$ that has linear margin $\tau$ as in \Cref{def:margin}. 
    Then there is a $\rho$-replicable algorithm
    such that given $m = \wt{O}(\eps^{-1} \tau^{-4} \rho^{-2} \log(\nicefrac1\delta))$ i.i.d. samples $(x,y) \sim \calD$, 
    computes in time $\poly(d) \cdot \poly(1/\eps, 1/\tau, 1/\rho, 1/\delta)^{1/\tau^2}$, 
    a {unit} vector $w \in \reals^d$ satisfying
    $\Pr_{(x,y) \sim \calD}[\sgn(w^\top x) \neq y] \leq \eps$
    with probability at least $1-
    \delta$.
\end{theorem}

Compared to the DP-to-Replicability reduction, \Cref{thm:main3} has better dependence on $1/\eps$ and better running time.
For the proof of this result, we refer to \Cref{sec:inefficient-algo}.

\subsection{Related Work}
\label{sec:priorwork}
{In terms of technique,
a related work is that of \cite{nissim2016locating} for differentially private clustering.
The authors also use the JL projection followed by a discretization step as part of their algorithm. 
Similar to our techniques,
this serves to avoid a factor in the discretization scheme that scales with the dimension.
One important difference is that \cite{nissim2016locating} solves a task on the samples while we need to solve a task on a distribution using samples. 
Moreover,
while \cite{nissim2016locating} also use a discretization step (not based on the Alon-Kartag scheme) to find a ``dense'' part of the space to use as the center of the cluster,
the discretization step in our work serves a different purpose:
to ensure that our estimates, across two i.i.d. executions, will be the same.}

\paragraph{Replicability.}
Pioneered by \citet{impagliazzo2022reproducibility},
there has been a growing interest from the learning theory community in studying
replicability as an algorithmic property in various
learning tasks. Among other things, 
their work showed that the fundamental class of statistical queries, which appears in various settings (see e.g., \citet{blum2003noise,gupta2011privately,goel2020statistical,fotakis2021efficient} and the references therein) can be made replicable. 
Subsequently, \citet{esfandiari2023replicableb, esfandiari2023replicable} studied 
replicable algorithms in the context of multi-armed bandits and clustering. Later, \citet{eaton2023replicable, karbasi2023replicability} studied replicability 
in the context of Reinforcement Learning.
Recently, \citet{bun2023stability,kalavasis2023statistical,kalavasis2024computational} established equivalences between replicability
and other notions of algorithmic stability such as differential privacy (DP),
and \citet{moran2023bayesian} derived more fine-grained characterizations
of these equivalences. 
It is worth mentioning that \citet{malliaris2022unstable} had already established
equivalences between various notions of algorithmic stability
and finiteness of the Littlestone dimension of the
underlying concept class. Inspired by \citet{impagliazzo2022reproducibility},
a related line of work
\citep{chase2023replicability, dixon2023list,chase2023local} proposed and studied alternative notions of replicability such as
list-replicability, where the requirement is that
when the algorithm is executed multiple times
on i.i.d. datasets, then the number of different solutions
it will output across these executions is small.

\paragraph{Large-Margin Halfspaces.} The problem of learning large-margin halfspaces has been extensively studied and has inspired various fundamental algorithms \citep{rosenblatt1958perceptron,vapnik1999nature,freund1997decision,freund1998large}. In the DP setting,
\citet{blum2005practical} gave a dimension-dependent construction based on a private version of the perceptron algorithm. This was later improved by \citet{le2020efficient}
who gave new DP algorithms for this task with dimension-independent guarantees based on the Johnson-Lindenstrauss transformation.
Next, \citet{bun2020efficient} constructed noise-tolerant and private
PAC learners for large-margin halfspaces whose sample complexity also does not depend on the
dimension.
\citet{beimel2019private} and \citet{kaplan2020private} designed private algorithms for
learning halfspaces without margin guarantees when the domain is finite. \citet{bassily2022open} stated an open problem of finding optimal DP algorithms for learning large-margin halfspaces both with respect to their running time and their sample complexity. \citet{bassily2022differentially} studied DP algorithms for various learning tasks with margin, including halfspaces, kernels, and neural networks.
In the area of robust statistics, \citet{diakonikolas2023information} showed a statistical-computational tradeoff in the problem of PAC learning large-margin halfspaces with random classification noise. For further results on robustly learning large-margin halfspaces, we refer to \citet{diakonikolas2019nearly} and the references therein.

\section{The Main Tool: The Alon-Klartag Rounding Scheme}
\label{section:rounding}
Inspired by \citet{alon2017optimal}, we introduce and use the following rounding scheme $\AKround(z, \beta)$ for a point $z$ with parameter $\beta$:
let $o = (o_1, \dots, o_k) \sim_{i.i.d.} U[0, \beta]$ be uniformly random offsets
and implicitly discretize $\R^k$ using a grid of side length $\beta$ centered at $o$.
Let $o(z)\in \R^k$ denote the ``bottom-left'' corner of the cube in which $z$ lies{, i.e., the point obtained by rounding down all the coordinates of $z$}.
For a vector $z$, we let $z[i]$ be its $i$-th coordinate.
Define $p(z)[i]\in [0, 1]$ to be such that 
\[
    p(z)[i] \cdot o(z)[i] + (1-p(z)[i])\cdot (o(z)[i] + \beta) = z[i]\,.
\]

Given offsets $o$ and thresholds
$u = (u_1,\dots,u_k)$ with $u_i \sim U[0,1]$, round a vector $z$ to $f_{o,u}(z)$ where the $i$-th coordinate is equal to $o(z)[i]$ if $u_i \leq p(z)[i]$ and $o(z)[i] + \beta$ otherwise.
{Crucially, in expectation, the rounded point $f_{o,u}(z)$ coincides with $z$.}

The next lemma,
whose proof can be found in \Cref{apx:rounding},
is useful in order to derive the replicability guarantees of our
rounding scheme.

\begin{restatable}[Stability of Rounding]{lemma}{roundingReplicable}
\label{lemma:round}
    Let $z, z' \in \R^k$. Then for independent uniform offsets $o_1,\dots,o_k \in [0,\beta]$ and thresholds $u_1,\dots,u_k \in [0,1]$, we have 
    $
    \Pr_{o,u}[f_{o,u}(z) \neq f_{o,u}(z')] \leq 2 \beta^{-1} \|z-z'\|_1.
    $
\end{restatable}

Our novel analysis of the stability of the rounding scheme under shared randomness (cf. \Cref{lemma:round})
demonstrates its useful properties for designing replicable algorithms.
We believe these properties may be of interest beyond the scope of this work and can find applications
in designing replicable algorithms for different problems.

Next, we show that the Alon-Klartag rounding scheme additively preserves inner products with high probability. This is formalized below in a lemma whose proof can be found in \Cref{apx:rounding}.

\begin{restatable}[Rounding preserves Inner Products]{lemma}{roundingInnerProduct}
    \label{lemma:inner}
    Let $z, x \in \R^k$ be such that $\ \|x\|\leq 1$.
    For uniform offsets $o_1,\dots,o_k \in [0,\beta]$ and thresholds $u_1,\dots,u_k \in [0,1]$, we have 
    $
    \Pr_{o,u}\left[ | f_{o,u}(z)^\top  x -  z^\top x| > \alpha \right] \leq 2 \exp(-2 \alpha^2 \beta^{-2}).
    $   
\end{restatable}

It is worth mentioning that the
Alon-Klartag rounding scheme, along with
dimensionality reduction techniques, 
was also used by \citet{gronlund2020near}
in order to prove generalization bounds 
for SVMs.

\section{Replicably Learning Large-Margin Halfspaces: \Cref{alg:algo2}}
\label{sec:algo2}
In this section, we describe our first algorithm and prove its guarantees as stated in \Cref{thm:main2}. 
Let $\cD$ be a distribution over $\R^d \times \{-1,1\}$ 
with linear margin {$\tau\in (0, 1)$} (cf. \Cref{def:margin}).
Thus, there is a unit vector $w^\star \in \R^d$ such that
for all $(x,y) \in \supp(D), x \neq 0$, we have $y (x^\top w^\star /\|x\|) \geq \tau$. 
Given $\eps, \rho, \delta\in (0, 1)$, 
our goal is to design a $\rho$-replicable learning algorithm that draws $m = m(\eps,\tau,\rho,\delta)$ i.i.d. samples from $\calD$
and outputs $\hat{w}\in \R^d$ such that, with probability at least $1-\delta$
over the randomness of the samples and (potentially) the internal randomness of the algorithm, it holds that $\Pr_{(x,y) \sim \cD}[y(\hat{w}^\top x) \leq 0] \leq \eps$.

\paragraph{Description of \Cref{alg:algo2}.}
We consider $B$ batches of $n$ samples each. Hence, in total, we draw $n B$ i.i.d. samples from $\calD$.
On each batch $i \in [B]$, we run the standard SVM algorithm (cf. \Cref{lemma:svm-margin}) to find a hyperplane with normal vector $w_i \in \reals^d$ that has margin at least $\tau/2$ on all training data in the batch. 
We then compute the average normal vector $z = (1/B) \sum_{i \in [B]} w_i$. Finally, we round $z$ as described in \Cref{section:rounding}. 

\begin{algorithm}[H]
\caption{Replicable Large-Margin Halfspaces}\label{alg:algo2}
\begin{algorithmic}[1]
\STATE {$k \gets C_1 \tau^{-2} \log(\nicefrac1{\eps\tau\rho\delta})$}
\STATE {$B \gets C_2 \tau^{-4} \rho^{-2} \log(\nicefrac1{\eps\tau\rho\delta})$}
\STATE {$n \gets C_3 \eps^{-1} \tau^{-3} \log(\nicefrac1{\eps\tau\rho\delta})$}
\STATE {$\beta \gets C_4 \tau/\log(\nicefrac1{\eps\tau\rho\delta})$}
    
\FOR {$i = 1,2,...B$}
    \STATE $S_i = $ batch of $n$ i.i.d. samples from $\calD$
    \STATE $w_i \gets \SVM(S_i, \tau/2)$
\ENDFOR
  
\STATE $z \gets (1/B) \sum_{i \in [B]} w_i$
\STATE Draw $A \in \reals^{k \times d}$ with $A_{i,j} \sim \calN(0,1/k)$ %
\hfill {(shared randomness)}

\STATE $b\gets \AKround(Az, \beta)$ %
\hfill {(shared randomness, cf. \Cref{section:rounding})}
\STATE \textbf{return} $\hat{w} = A^\top b / \|A^\top b\|$ 
\end{algorithmic}
\end{algorithm}

\paragraph{Correctness of \Cref{alg:algo2}.}
\sloppy
A straightforward adaptation of the results of \citet{gronlund2020near} (cf. \Cref{lemma:svm-margin}) shows
that, with probability {$1-\delta/(10 B)$} over the samples, 
the classifier $w_i$ has margin at least $\tau/4$ in a $\left(1-O \left(\frac{\log n + \log(B/\delta)}{\tau^2n} \right) \right)$
fraction of the population, i.e., 
\[
    \Pr_{(x,y)\sim \calD}[y(w_i^\top x)/\|x\| < \tau/4] \leq O\left(\frac{\log n + \log(B/\delta)}{\tau^2n}\right) \,.
\]
We denote the complement of this event as $E_i$ 
and condition on not observing $\cup_{i \in [B]} E_i$.

Now under this event,
for all the vectors $w_i, i \in [B]$ and all points $(x,y) \in \reals^d \times \{-1,1\}$, it holds that $y(w_i^\top x)/\|x\| \geq -1,$
since $w_i$ is a unit vector.
Furthermore,
for a $(1- \widetilde O(\nicefrac{1}{\tau^2 n}))$-fraction of the population,
the margin is at least $\tau/4$, i.e., $y(w_i^\top x)/\|x\| \geq \tau/4.$
Intuitively, this means that
the vector $z = \frac{1}{B } \sum_{i \in [B]} w_i$ should have margin at least $\tau/8$, 
except for an $\wt{O}(1/(\tau^3 n))$ fraction of the population.
Formally, 
for $(x,y)\sim \calD$,
let $Z_i$ be the indicator variable such that $y\cdot (w_i^\top x) / \norm{x} < \tau/4$
and $Z := \sum_{i\in [B]} Z_i$.
Then
\begin{align*}
    &y(z^\top x)/\|x\| \\
    &= y\left( \frac{1}{B} \sum_{i \in [B]} w_i\right)^\top (x/\|x\|)
    = \frac{1}{B} \sum_{i \in [B]} y\cdot w_i^\top(x/\|x\|) \\
    &\geq \frac{1}{B} \left(-Z + (B-Z)\tau/4\right)
    = \tau/4 - \frac{Z}{B}(1+\tau/4) \,.
\end{align*}
This means that if $ y(z^\top x)/\|x\| < \tau/8$, then
\begin{align*}
     \tau/4 - \frac{Z}{B}(1+\tau/4) < \tau/8 \implies
     Z > \frac{\tau}{16} \cdot B \,.
\end{align*}
It suffices to bound the probability of the event that $Z > \Omega(\tau B)$ to bound the population error of $z$.

Notice that the summation of the fractions 
of the population where the $w_i$ have margin less than $\tau/4$ is at most $O(B(\log n + \log(\nicefrac{B}\delta))/(\tau^2n))$.
As noted above,
at least $\Omega(\tau B)$ of the classifiers must simultaneously have margin less than $\tau/4$ for $z$ to misclassify $x$.
Thus the fraction of the population where $z$
has margin smaller than $\tau/8$ is at most
\begin{align*}
    O\left(\frac{B(\log n + \log(\nicefrac{B}\delta))}{\tau \cdot B (\tau^2n)}\right)
    &= O\left( \frac{\log n + \log(\nicefrac{B}\delta)}{\tau^3 n} \right) \,.
\end{align*}
Thus, choosing
$
    n \geq \wt{\Omega}(\log(\nicefrac{B}\delta)/(\tau^3 \eps))
$
ensures that the intermediary normal vector $z = (1/B) \sum_{i=1}^B w_i$ satisfies $\Pr_{(x, y)\sim \calD} [y(z^\top x/\norm{x}) < \tau/8] \leq \eps/10$ with probability at least $1-\delta/10$.

The following lemma whose proof is deferred to \Cref{apx:algo2}
ensures that projecting and rounding in the lower dimension
approximately preserves the performance of $z$ with respect to the 0-1 loss
(as opposed to the $\tau/8$-loss).
\begin{restatable}{lemma}{projectingRoundingMargin}\label{lem:projecting rounding preserves margin}
    Fix $\eps, \tau, \delta\in (0, 1)$
    and let $\calD$ be a distribution over $\R^d\times \set{\pm 1}$
    that admits a linear classifier with $\tau$-margin.
    Suppose $z$ is a random unit vector satisfying
    $
        \Pr_{(x, y)\sim \calD} [y(z^\top x/\norm{x}) < \tau/2] \leq \eps
    $
    with probability at least $1-\delta$ over the draw of $z$.
    Define $k = \Omega(\tau^{-2}\log(\nicefrac1{\eps\delta}))$
    and $\beta = O(\tau/\log(\nicefrac1{\eps\delta}))$.
    If $A\in \R^{k\times d}$ is a JL-matrix (cf. \Cref{proof:JL}) and $b = \AKround(Az, \beta)$ (cf. \Cref{section:rounding}),
    then $\hat w = A^\top b/ \norm{A^\top b}$ satisfies
    $
        \Pr_{(x, y)\sim \calD} [y(\hat w^\top x/\norm{x}) \leq 0] \leq 2\eps
    $
    with probability at least $1-2\delta$.
\end{restatable}

We note that an application of it with $\tau' = \tau/4$,
$\eps' = \eps/10$,
and $\delta' = \delta/10$
yields that the final output $\hat w$ of \Cref{alg:algo2} has 0-1 population error of at most $\eps/5$ with probability at least $1-\delta/5$ as desired.

\paragraph{Replicability of \Cref{alg:algo2}.}
We now state the lemma from which the replicability guarantees 
of our algorithm follow.
Its proof is deferred to \Cref{apx:algo2}.

\begin{restatable}{lemma}{projectingRoundingReplicable}\label{lem:projecting rounding is replicable}
    Fix $\eps, \tau, \rho, \delta\in (0, 1)$.
    Suppose $w_1, \dots, w_B$ and $w_1', \dots, w_B'$ are i.i.d. random unit vectors %
    for $B = \Omega(\tau^{-4}\rho^{-2}\log(\nicefrac1{\eps\tau\rho\delta}))$
    and $z = (1/B) \sum_{i\in [B]} w_i$,
    $z' = (1/B) \sum_{i\in [B]} w_i'$ are their averages.
    Define $k = \Theta(\tau^{-2}\log(\nicefrac1{\eps\tau\rho\delta}))$
    and $\beta = \Theta(\tau/\log(\nicefrac1{\eps\tau\rho\delta})$.
    If $A\in \R^{k\times d}$ is a JL-matrix (cf. \Cref{proof:JL})
    and $b = \AKround(Az, \beta)$, $b' = \AKround(Az', \beta)$ (cf. \Cref{section:rounding}),
    then $b = b'$ with probability at least $1-\rho$ over the draw of the $w_i', w_i$'s, $A$, and $\AKround$.
\end{restatable}

Note that the $w_i$'s are i.i.d. unit vectors across all batches and two independent executions since the samples in the $2B$ batches in the two executions are drawn from the same distribution $\calD$
and the output of the SVM algorithm depends only on its input sample. 
Hence an application of \Cref{lem:projecting rounding is replicable} ensures that \Cref{alg:algo2} is indeed $\rho$-replicable.

\paragraph{Sample Complexity \& Running Time of \Cref{alg:algo2}.} 
The sample complexity is $nB = \tilde O(\eps^{-1} \tau^{-7} \rho^{-2} \log(\nicefrac1\delta))$.
By inspection,
we see that the total running time of \Cref{alg:algo2} is $\poly(d, n) = \poly(d, 1/\eps, 1/\tau, 1/\rho, \log(\nicefrac1\delta))$.

\section{Replicably Learning Large-Margin Halfspaces: \Cref{alg:algo4}}\label{sec:algo4}
Let $B_1^d$ denote the unit $\ell_2$-ball in $d$-dimensions.
Our approach is inspired by the work of \cite{le2020efficient} that designed a similar SGD approach for learning large-margin halfspaces under differential privacy constraints.  
Consider the following surrogate loss $h$
\begin{align*}
  B_1^d\times B_1^d\times \set{-1, +1} &\to \R_{\geq 0} \\
  h(w; x, y)
  &:= \max\left( 0, 2-\frac2\tau y(x^\top w) \right) \\
  &\geq \ones\set{y(x^\top w) < \tau/2}.
\end{align*}
We remark that $h(w; x, y) \geq 1$ when $y(x^\top w) \leq \tau/2$
and $h(w; x, y) = 0$ when $y(x^\top w) \geq \tau$.
Also,
since $x, w\in B_1^d$,
an application of the Cauchy-Schwartz inequality
reveals that $h(w; x, y)\in [0, 2+2/\tau]$.
Finally,
$h$ is piecewise linear with each piece being $2/\tau$-Lipschitz.
Hence, $h$ is $O(1/\tau)$-Lipschitz.

\paragraph{Description of \Cref{alg:algo4}.}
Fix $\eps\in (0, 1)$.
We seek to minimize the following loss function over the ball $B_1^d$:
\[
  f_{\calD}(w)
  := \E_{(x, y)\sim \calD} [h(w; x, y)] + \frac{\eps}{10} \norm{w}^2.
\]
By construction,
the co-domain of $f_\calD$ lies within $[0, 2+2/\tau+\eps/10]\sset [0, O(1/\tau)]$.
Note also that $f_\calD$ is an upper bound on the $\tau/2$-population loss, i.e.,
$
\Pr_{(x,y) \sim \calD}
[y(x^\top w) < \tau/2]
\leq f(w)\,.
$
First, regarding the minima of $f_\calD$, note that any vector $w \in B_1^d$ achieving a margin of $\tau$ satisfies $f_{\calD}(w)\leq \eps/10$.
This is because $h(w; x, y) = 0$ for all $(x, y)\in \supp(\calD)$.
As a result,
$\min_{w \in B_1^d} f_\calD(w) \leq \eps/10$.
Second, let us consider an $\eps/10$-optimal solution $w'$ with respect to $f_{\calD}$, i.e.,
$
f_\calD(w') - \min_{w \in B_1^d} f \leq \eps/10\,.
$
The above discussion implies that $f_{\calD}(w') \leq \eps/5$
and is thus a $\tau/2$-margin classifier for an $\eps/5$-fraction of the population, i.e., $w'$ satisfies
$
\Pr_{(x,y) \sim \calD}
[y(x^\top w') < \tau/2] \leq \eps/5\,.
$

Note that we may assume without loss of generality that the marginal of $\calD$ over features is supported over $B_1^d$ since we normalize the input $x\mapsto x/\norm{x}$ before applying a classifier $w$.

Since $f$ is $O(\eps + 1/\tau) = O(1/\tau)$-Lipschitz and $\Omega(\eps)$-strongly convex,
we can apply the following standard result: %
\begin{theorem}[Theorem 6.2 in \citet{bubeck2015convex}]\label{thm:sgd convergence}
  Let $f$ be $\mu$-strongly convex with minimizer $w_*$ and assume that the (sub)gradient oracle $g(w)$ satisfies $\E[\norm{g(w)}^2]\leq G^2$.
  Then after $T$ iterations,
  projected stochastic gradient descent $\SGD(\calD, f, T)$ with step size $\eta_t = \nicefrac2{\mu(t+1)}$ satisfies
  \[
    \E\left[ f\left( \sum_{t=1}^T \frac{2t}{T(T+1)} w_t \right)\right] - f(w_*)
    \leq \frac{2G^2}{\mu(T+1)}.
  \]
\end{theorem}
Since $G^2 = O(1/\tau^2)$, $\mu = \Omega(\eps)$ and $f(w_*) \leq \eps/10$,
choosing $T\geq \Omega(\eps^{-2}\tau^{-2})$ yields an $\eps/10$-optimal solution in expectation.
Repeating this process independently for a small number of times
and outputting the one with the lowest objective yields an $\eps/5$-optimal solution with high probability.

\begin{restatable}{lemma}{boostedSGD}\label{lem:boosted SGD}
Let $B_1^d$ be the unit $\ell_2$-ball in $d$
 dimensions. Fix $\eps, \tau, \delta\in (0, 1)$
    and let $\cal D$ be a distribution over $B_1^d\times \set{\pm 1}$
    that admits a linear $\tau$-margin classifier.
    There is an algorithm $\boostSGD(\calD, \eps, \tau, \delta)$
    that outputs a unit vector $\tilde w\in \R^d$
    such that $f_{\calD}(\tilde w)\leq \min_{w\in B_1^d} f_\calD(w) + \eps$ with probability at least $1-\delta$.
    Moreover,
    $\boostSGD$ has sample complexity $\tilde O(\eps^{-2}\tau^{-2} \log(\nicefrac1\delta))$
    and running time $\mathrm{poly}(\nicefrac{1}{\eps}, \nicefrac{1}{\tau}, \log(\nicefrac{1}{\delta}), d)$.
\end{restatable}

The proof of \Cref{lem:boosted SGD} is deferred to \Cref{apx:algo4}.
Next,
we repeat $\boostSGD$ and take an average to ensure concentration before proceeding as in \Cref{alg:algo2} with the random projection and rounding in the lower dimensional space.
Compared to \Cref{alg:algo2},
we obtain an improved sample dependence on $\tau$ for \Cref{alg:algo4} since taking an average of $\eps$-optimal solutions to a convex objective function yields an $\eps$-optimal solution.
However,
we pay an extra factor of $\eps$ in order to run $\SGD$ as a subroutine.

\begin{algorithm}[H]
\caption{Replicable Large-Margin Halfspaces}\label{alg:algo4}
\begin{algorithmic}[1]
\STATE $n = C_1 \eps^{-2} \tau^{-2} \log(\nicefrac1{\eps\tau\rho\delta})$
\STATE $B \gets C_2 \tau^{-4}\rho^{-2} \log(\nicefrac1{\eps\tau\rho\delta})$
\STATE $k \gets C_3\tau^{-2}\log(\nicefrac1{\eps\tau\rho\delta})$
\STATE $\beta \gets C_4\tau/\log(\nicefrac1{\eps\tau\rho\delta})$

\FOR {$i\gets 1, \dots, B$}
    \STATE $\tilde S_i \gets$ $n$ samples from $\cal D$
    \STATE $S_i \gets \set{(x/\norm{x}, y): (x, y)\in \tilde S_i}$
    \STATE $w_i \gets \boostSGD(S_i, \eps/10, \tau, \delta/B)$ %
    \hfill (cf. \Cref{lem:boosted SGD})
\ENDFOR  
\STATE $z \gets (1/B) \sum_{i \in [B]} w_i$
\STATE Draw $A \in \reals^{k \times d}$ with $A_{i,j} \sim \calN(0,1/k)$ %
\hfill {(shared randomness)}
\STATE $b\gets \AKround(Az, \beta)$ %
\hfill {(shared randomness, cf. \Cref{section:rounding})}
\STATE \textbf{return} $\hat{w} = A^\top b / \|A^\top b\|$ 
\end{algorithmic}
\end{algorithm}

The technical details of \Cref{alg:algo4} (the proof of \Cref{thm:main4}) appear in \Cref{proof:algo4}.

\section{Replicably Learning Large-Margin Halfspaces: \Cref{alg:algo3-inefficient}}\label{sec:inefficient-algo}
\sloppy
In this section, we provide an algorithm whose sample complexity scales
as $\widetilde{O}(\eps^{-1}\tau^{-4}\rho^{-2}\log(\nicefrac1\delta))$ and has running time $\poly(d) \cdot (\poly(1/\eps, 1/\rho, 1/\tau, 1/\delta))^{1/\tau^2}$. 
We remark that the sample complexity is better than 
the one obtained from the DP transformation in \Cref{prop:dp-reduction-sample-complexity}.
Moreover, the running time is exponentially better than that obtained through the DP
transformation.

Before this,
we state a very useful result due to \citet{bun2023stability} related to the sample
complexity of replicably learning finite hypothesis classes, which is used in the proof of \Cref{thm:main3}.
\begin{proposition}
[Theorem 5.13 in \citet{bun2023stability}]
\label{prop:bun}
Consider a finite concept class $\calH$. There is a $\rho$-replicable agnostic PAC learner $\rLearnerFinite$ with accuracy $\epsilon$ and confidence $\delta$ for $\calH$ with sample complexity $n = O\left( \frac{\log^2|H| + \log(\nicefrac1{\rho \delta})}{\epsilon^2 \rho^2} \log^3(\nicefrac1\rho)\right)$. Moreover, if $\calD$ is realizable then the sample complexity drops to $\wt{O}(\eps^{-1} \rho^{-2} \log^2|H|)$. Finally, the algorithm terminates in time $\poly(|\calH|,n).$
\end{proposition}

\paragraph{Description of \Cref{alg:algo3}}
Let us first provide the algorithm's pseudo-code.

\begin{algorithm}[H]
\caption{Replicable Large-Margin Halfspaces}\label{alg:algo3-inefficient}
\label{alg:algo3}
\begin{algorithmic}[1]
    \STATE $k \gets C \tau^{-2} \log(\nicefrac1{\eps \tau \rho \delta})$
    \STATE $n \gets C' \eps^{-1}\tau^{-4}\rho^{-2}\log(\nicefrac1{\eps \rho \tau \delta})$
    
    \STATE  $S \gets $ batch of $n$ i.i.d. samples from $\calD$  
    \STATE Draw $A \in \reals^{k \times d}$ with $A_{i,j} \sim \calN(0,1/k)$ %
    \hfill {(shared randomness)}
    \STATE  $S_A \gets (Ax_1,y_1),\ldots,(Ax_n,y_n)$
    \STATE  $\calH_\tau \gets $ a $(\tau/20)$-net over vectors of  length at most 1 in $\reals^k$ 
    \STATE $b \gets $ output of $\rLearnerFinite$ from \Cref{prop:bun} with input $S_A, \calH_\tau$
    \hfill {(shared randomness)}
    \STATE  $\hat{w} \gets A^\top b / \|A^\top b\|$
\end{algorithmic}
\end{algorithm}

Similar to the other algorithm, we first start
by using a JL matrix to project the training set to a $k$-dimensional space,
for $k = \widetilde{\Theta}(\tau^{-2}).$ Then, we use a $(\tau/20)$-net to cover
all the unit vectors of this $k$-dimensional space, so the size
of the net is $\left(\nicefrac{C''}{\tau}\right)^{\wt{O}(\tau^{-2})}$, for 
some absolute constant $C'' > 0.$
We think of these points of the net as our hypothesis class $\calH_\tau$.
By the properties of the JL transform,
we can show that with high probability $1-\delta/10$, 
there exists a vector
in this class that classifies the entire training set correctly.
Moreover, we can show that this classifier has small generalization error.
This is formalized in the following result,
which is essentially a high-probability
version of \Cref{lemma:JL-imply} from \citet{le2020efficient}.
\begin{restatable}{lemma}{whpGoodChoiceJL}
\label{lem:whp-good-choice-JL-mapping}
    Fix $\eps', \delta_{JL}\in (0, 1)$.
    Let $\calD$ satisfy \Cref{def:margin} with margin $\tau$
    and suppose $w^\star$ satisfies $y(w^\star)^\top x \geq \tau$ for every $(x, y)\in \supp(\cal D)$.
    For a JL-matrix $A \in \reals^{k \times d}$, as stated in \Cref{lemma:distr-JL} with $k = \Omega(\tau^{-2} \log(1/\delta_{JL}))$, 
    let $G_A \subseteq \reals^d \times \{-1,1\}$ be the set of points $(x,y)$ of $\supp(\cal D)$ that satisfy
    \begin{itemize}
        \item $\abs*{\|Ax\|^2 - \|x\|^2} \leq \tau \|x\|^2/100,$ and,
        \item 
        $y (A w^\star/\|A w^\star\|)^\top (A x/\|A x\|) \geq 96 \tau/100$.
    \end{itemize}
    Let $E_1$ be the event (over $A)$ that $\Pr_{(x,y)\sim \calD}[(x,y)\in G_A] \geq 1 - \eps'$
    and $E_2$ be the event (over $A$) that $\abs*{\|Aw^\star\|^2 - \|w^\star\|^2} \leq \tau \|w^\star\|^2/100.$
    Then it holds that {$\Pr_{A}[E_1 \cap E_2] \geq 1-\delta_{JL}/\eps'$}.
\end{restatable}
The proof of \Cref{lem:whp-good-choice-JL-mapping} appears in \Cref{apx:whp-good-choice-JL-mapping}.
One way to view \Cref{lem:whp-good-choice-JL-mapping} is that, with probability $1-\delta_{JL}/\eps'$ over the random choice
of $A$, the classifier $Aw^\star/\|Aw^\star\|$ will have $96 \tau/100$ margin
on a $1-\eps'$ fraction of $\calD,$ where
the choice of $\eps'$ will be specified
later according to \Cref{lem:whp-good-choice-JL-mapping}.
Let us condition on this event for the rest 
of the proof, which we call $E_r.$ 
Let $\tilde w^\star$ be the point of the net that $Aw^\star/\|Aw^\star\|$ is rounded to. {Recall that we round to the closest point on a $\tau/20$-net of the unit ball with respect to the $\ell_2$-norm, 
so that $\tilde w^\star$ is $\tau/20$ close to the normalized version of $A w^\star$}.
Notice that under the event $E_r$,
for all points $(x,y) \in G_A$ we have that 
\begin{align*}
    &(\tilde w^\star)^\top \frac{A x}{\|A x\|} \\
    &= \frac{(A w^\star)^\top}{\|A w^\star\|} \frac{A x}{\|A x\|} 
    - \left( \frac{Aw^\star}{\|A w^\star\|} - \tilde w^\star \right)^\top \frac{A x}{\|A x\|} \\
    &\geq \frac{(A w^\star)^\top}{\|A w^\star\|} \frac{A x}{\|A x\|} 
    - \norm*{\frac{Aw^\star}{\|A w^\star\|} - \tilde w^\star} \cdot \frac{\|Ax\|}{\|Ax\|} \\
    &\geq 96\tau/100 - \tau/20 > 9/(10 \tau)\,,
\end{align*}
where the first inequality follows
from Cauchy-Schwartz and the second inequality
from the definition of the net.
Since the hypothesis class $\calH_\tau$
has finite size,
we can use \Cref{prop:bun} from \citet{bun2023stability}
which states that $\widetilde{O}(\eps^{-1} \rho^{-2} \log(\nicefrac1\delta) \log^2|\calH|)$ samples are sufficient
to $\rho$-replicably learn a hypothesis class $\calH$ in the \emph{realizable} setting
with error at most $\eps.$ 
One technical complication we need to handle is that $\calD$ is not necessarily realizable
with respect to $\calH_\tau.$
Nevertheless, under $E_r,$ we have shown that for $\tilde w^\star \in \calH_\tau$ it holds that
$\Pr_{(x,y) \sim \calD}[y((\tilde w^\star)^\top Ax/\|Ax\|) < 9\tau/10] \leq \eps'.$
Let us denote by $\calD_r$ the distribution $\calD$ conditioned on the event 
that $y((\tilde w^\star)^\top Ax/\|Ax\|) \geq 9\tau/10$ and $\calD_b$ its complement,
i.e., $\calD$ conditioned on  $y((\tilde w^\star)^\top Ax/\|Ax\|) < 9\tau/10.$
Then, we can express $\calD = (1-p_b)\cdot \calD_r + p_b\cdot \calD_b,$ where
$p_b \leq \eps'.$ Hence,
in a sample of size $n$ from $\calD$,
with probability at
least $1-n\cdot \eps'$
we only see samples drawn
i.i.d. from $\calD_r$. Let us call
this event $E_r'$ and condition on it.
{Let us choose $\eps' = \delta/(10n),
\delta_{JL} = \eps'^2/10 = \delta^2/1000n$ so that $\Pr[E_r\cap E_r'] \geq 1-\delta$.}

Under the events $E_r, E_r',$ we can use
the replicable learner from \citet{bun2023stability}
to learn the \emph{realizable} distribution
$\calD_r.$ Run
the algorithm from \Cref{prop:bun} with 
parameters $\delta/10,\eps/10,\rho/10$.\footnote{We assume that $\delta < \rho, \eps$, otherwise
we normalize $\delta$ to get the bound.}
Since
$|\calH| = \left(\nicefrac{C}{\tau}\right)^{\tau^{-2}}$
for some absolute constant $C$,
we see 
that we need $n = \widetilde O\left(\eps^{-1} \tau^{-4} \rho^{-2} \log(\nicefrac1\delta) \right)$
samples.

\paragraph{Replicability of \Cref{alg:algo3}.}
Let us condition on the events $E_r, E_r'$ across
two executions of the algorithm. Since,
the matrix $A$ is the same, due to shared
randomness, under these events
the samples that the learner
of \citet{bun2023stability} receives
are i.i.d. from the same \emph{realizable}
distribution.
Then, the replicability guarantee follows
immediately from the guarantees of
\Cref{prop:bun} and the fact that all
the events we have conditioned on
occur with probability at least $1-\rho/2.$

\paragraph{Correctness of \Cref{alg:algo3}.}
{Here we examine the population error of the output of \Cref{prop:bun}.
Note that this error is analyzed with respect to the original distribution $\calD$,
which is not necessarily realizable,
rather than $\calD_r$,
which is the distribution on which we run the learning algorithm.}

Let us again condition on the events 
$E_r, E_r',$ and the event that the
output of \Cref{prop:bun} satisfies
the generalization bound it states.
We assume without loss of generality that $\delta \leq \eps/2,$ otherwise we set $\delta = \eps/2$ without
affecting the overall sample complexity of our algorithm.
Notice that these events occur with probability
at least $1-\delta.$ Let $b$ be the output
of the algorithm.
Then we have
\begin{align*}
    \Pr_{(x,y) \sim \calD}[y (b^\top x) < 0]
    &= (1-p_b) \Pr_{(x,y) \sim \calD_r}[y (b^\top x) < 0]  \\
    &\qquad + p_b \Pr_{(x,y) \sim \calD_b}[y (b^\top x) < 0] \\
    &\leq (1-p_b) \eps/10 + p_b
    < \eps.
\end{align*}
This concludes the proof.

\paragraph{Sample Complexity \& Running Time of \Cref{alg:algo3}.} 
As noted before,
we need $n = \widetilde O\left(\eps^{-1} \tau^{-4} \rho^{-2} \log(\nicefrac1\delta) \right)$
samples.
Since we apply \Cref{prop:bun},
we incur a running time of $\poly(d) \cdot \poly(1/\epsilon,1/\rho,1/\tau,1/\delta)^{1/\tau^2}.$

\section{Conclusion}\label{sec:conclusion}
In this work, we have developed new algorithms for replicably learning large-margin halfspaces.
Our results vastly improve upon prior work on this problem. 
We believe that
many immediate questions for future research arise from our work.
First, it is natural to ask whether there are efficient algorithms that can achieve
the $\widetilde{O}(\eps^{-1} \tau^{-4} \rho^{-2}\log(\nicefrac1\delta))$ sample complexity bound 
of \Cref{alg:algo3-inefficient}. 
Also, it would be interesting to see if there are any (not necessarily efficient) replicable algorithms 
whose sample complexity scales as $\widetilde{O}(\eps^{-1} \tau^{-2} \rho^{-2})$ or if there
is some inherent barrier to pushing the dependence on $\tau$ below $\tau^{-4}.$ Finally, our analysis
of \Cref{alg:algo2} is pessimistic in the sense
that it uses a pigeonhole principle argument
to establish that the fraction of the population
where the aggregate vector does not have margin 
$\Omega(\tau)$ is $\widetilde O(\nicefrac{1}{\tau^3n}).$
It would be interesting to see whether this bound
can be improved to $\widetilde O(\nicefrac{1}{\tau^2n})$
using a different analysis, which would reduce
the overall dependence of the 
algorithm on $\tau.$

\section*{Acknowledgements}

Kasper Green Larsen is supported by Independent Research Fund Denmark (DFF) Sapere Aude Research Leader grant No 9064-00068B.
Amin Karbasi acknowledges funding in direct support of this work from NSF (IIS-1845032), ONR (N00014- 19-1-2406), and the AI Institute for Learning-Enabled Optimization at Scale (TILOS).  
Grigoris Velegkas is supported by TILOS, the Onassis Foundation, and the Bodossaki Foundation.
Felix Zhou is supported by TILOS.

\ifarxiv
\else
\section*{Impact Statement}
This paper presents work whose goal is to advance the field of Machine Learning. 
There are many potential societal consequences of our work, 
none of which we feel must be specifically highlighted here.
\fi

\bibliography{bib}

\begin{thebibliography}{49}
\providecommand{\natexlab}[1]{#1}
\providecommand{\url}[1]{\texttt{#1}}
\expandafter\ifx\csname urlstyle\endcsname\relax
  \providecommand{\doi}[1]{doi: #1}\else
  \providecommand{\doi}{doi: \begingroup \urlstyle{rm}\Url}\fi

\bibitem[Achlioptas(2001)]{achlioptas2001database}
Dimitris Achlioptas.
\newblock Database-friendly random projections.
\newblock In \emph{Proceedings of the twentieth ACM SIGMOD-SIGACT-SIGART
  symposium on Principles of database systems}, pages 274--281, 2001.

\bibitem[Alon and Klartag(2017)]{alon2017optimal}
Noga Alon and Bo'az Klartag.
\newblock Optimal compression of approximate inner products and dimension
  reduction.
\newblock In \emph{2017 IEEE 58th Annual Symposium on Foundations of Computer
  Science (FOCS)}, pages 639--650. IEEE, 2017.

\bibitem[Baker(2016)]{baker20161}
Monya Baker.
\newblock 1,500 scientists lift the lid on reproducibility.
\newblock \emph{Nature}, 533\penalty0 (7604), 2016.

\bibitem[Balcan and Feldman(2015)]{balcan2015statistical}
Maria{-}Florina Balcan and Vitaly Feldman.
\newblock Statistical active learning algorithms for noise tolerance and
  differential privacy.
\newblock \emph{Algorithmica}, 72\penalty0 (1):\penalty0 282--315, 2015.
\newblock \doi{10.1007/S00453-014-9954-9}.
\newblock URL \url{https://doi.org/10.1007/s00453-014-9954-9}.

\bibitem[Ball(2023)]{ball2023ai}
Philip Ball.
\newblock Is ai leading to a reproducibility crisis in science?
\newblock \emph{Nature}, 624\penalty0 (7990):\penalty0 22--25, 2023.

\bibitem[Bassily et~al.(2022{\natexlab{a}})Bassily, Mohri, and
  Suresh]{bassily2022differentially}
Raef Bassily, Mehryar Mohri, and Ananda~Theertha Suresh.
\newblock Differentially private learning with margin guarantees.
\newblock \emph{Advances in Neural Information Processing Systems},
  35:\penalty0 32127--32141, 2022{\natexlab{a}}.

\bibitem[Bassily et~al.(2022{\natexlab{b}})Bassily, Mohri, and
  Suresh]{bassily2022open}
Raef Bassily, Mehryar Mohri, and Ananda~Theertha Suresh.
\newblock Open problem: Better differentially private learning algorithms with
  margin guarantees.
\newblock In \emph{Conference on Learning Theory}, pages 5638--5643. PMLR,
  2022{\natexlab{b}}.

\bibitem[Beimel et~al.(2019)Beimel, Moran, Nissim, and
  Stemmer]{beimel2019private}
Amos Beimel, Shay Moran, Kobbi Nissim, and Uri Stemmer.
\newblock Private center points and learning of halfspaces.
\newblock In \emph{Conference on Learning Theory}, pages 269--282. PMLR, 2019.

\bibitem[Blum et~al.(2003)Blum, Kalai, and Wasserman]{blum2003noise}
Avrim Blum, Adam Kalai, and Hal Wasserman.
\newblock Noise-tolerant learning, the parity problem, and the statistical
  query model.
\newblock \emph{Journal of the ACM (JACM)}, 50\penalty0 (4):\penalty0 506--519,
  2003.

\bibitem[Blum et~al.(2005)Blum, Dwork, McSherry, and Nissim]{blum2005practical}
Avrim Blum, Cynthia Dwork, Frank McSherry, and Kobbi Nissim.
\newblock Practical privacy: the sulq framework.
\newblock In \emph{Proceedings of the twenty-fourth ACM SIGMOD-SIGACT-SIGART
  symposium on Principles of database systems}, pages 128--138, 2005.

\bibitem[Bubeck et~al.(2015)]{bubeck2015convex}
S{\'e}bastien Bubeck et~al.
\newblock Convex optimization: Algorithms and complexity.
\newblock \emph{Foundations and Trends{\textregistered} in Machine Learning},
  8\penalty0 (3-4):\penalty0 231--357, 2015.

\bibitem[Bun et~al.(2020)Bun, Carmosino, and Sorrell]{bun2020efficient}
Mark Bun, Marco~Leandro Carmosino, and Jessica Sorrell.
\newblock Efficient, noise-tolerant, and private learning via boosting.
\newblock In \emph{Conference on Learning Theory}, pages 1031--1077. PMLR,
  2020.

\bibitem[Bun et~al.(2023)Bun, Gaboardi, Hopkins, Impagliazzo, Lei, Pitassi,
  Sivakumar, and Sorrell]{bun2023stability}
Mark Bun, Marco Gaboardi, Max Hopkins, Russell Impagliazzo, Rex Lei, Toniann
  Pitassi, Satchit Sivakumar, and Jessica Sorrell.
\newblock Stability is stable: Connections between replicability, privacy, and
  adaptive generalization.
\newblock In \emph{Proceedings of the 55th Annual ACM Symposium on Theory of
  Computing}, pages 520--527, 2023.

\bibitem[Chase et~al.(2023{\natexlab{a}})Chase, Chornomaz, Moran, and
  Yehudayoff]{chase2023local}
Zachary Chase, Bogdan Chornomaz, Shay Moran, and Amir Yehudayoff.
\newblock Local borsuk-ulam, stability, and replicability.
\newblock \emph{arXiv preprint arXiv:2311.01599}, 2023{\natexlab{a}}.

\bibitem[Chase et~al.(2023{\natexlab{b}})Chase, Moran, and
  Yehudayoff]{chase2023replicability}
Zachary Chase, Shay Moran, and Amir Yehudayoff.
\newblock Stability and replicability in learning.
\newblock In \emph{64th {IEEE} Annual Symposium on Foundations of Computer
  Science, {FOCS} 2023, Santa Cruz, CA, USA, November 6-9, 2023}, pages
  2430--2439. {IEEE}, 2023{\natexlab{b}}.
\newblock \doi{10.1109/FOCS57990.2023.00148}.
\newblock URL \url{https://doi.org/10.1109/FOCS57990.2023.00148}.

\bibitem[Cortes and Vapnik(1995)]{cortes1995support}
Corinna Cortes and Vladimir Vapnik.
\newblock Support-vector networks.
\newblock \emph{Machine learning}, 20:\penalty0 273--297, 1995.

\bibitem[Diakonikolas et~al.(2019)Diakonikolas, Kane, and
  Manurangsi]{diakonikolas2019nearly}
Ilias Diakonikolas, Daniel Kane, and Pasin Manurangsi.
\newblock Nearly tight bounds for robust proper learning of halfspaces with a
  margin.
\newblock \emph{Advances in Neural Information Processing Systems}, 32, 2019.

\bibitem[Diakonikolas et~al.(2023)Diakonikolas, Diakonikolas, Kane, Wang, and
  Zarifis]{diakonikolas2023information}
Ilias Diakonikolas, Jelena Diakonikolas, Daniel~M Kane, Puqian Wang, and Nikos
  Zarifis.
\newblock Information-computation tradeoffs for learning margin halfspaces with
  random classification noise.
\newblock In \emph{The Thirty Sixth Annual Conference on Learning Theory},
  pages 2211--2239. PMLR, 2023.

\bibitem[Dixon et~al.(2023)Dixon, Pavan, Woude, and
  Vinodchandran]{dixon2023list}
Peter Dixon, Aduri Pavan, Jason~Vander Woude, and N.~V. Vinodchandran.
\newblock List and certificate complexities in replicable learning.
\newblock In Alice Oh, Tristan Naumann, Amir Globerson, Kate Saenko, Moritz
  Hardt, and Sergey Levine, editors, \emph{Advances in Neural Information
  Processing Systems 36: Annual Conference on Neural Information Processing
  Systems 2023, NeurIPS 2023, New Orleans, LA, USA, December 10 - 16, 2023},
  2023.

\bibitem[Dunagan and Vempala(2004)]{dunagan2004simple}
John Dunagan and Santosh~S. Vempala.
\newblock A simple polynomial-time rescaling algorithm for solving linear
  programs.
\newblock In L{\'{a}}szl{\'{o}} Babai, editor, \emph{Proceedings of the 36th
  Annual {ACM} Symposium on Theory of Computing, Chicago, IL, USA, June 13-16,
  2004}, pages 315--320. {ACM}, 2004.
\newblock \doi{10.1145/1007352.1007404}.
\newblock URL \url{https://doi.org/10.1145/1007352.1007404}.

\bibitem[Eaton et~al.(2023)Eaton, Hussing, Kearns, and
  Sorrell]{eaton2023replicable}
Eric Eaton, Marcel Hussing, Michael Kearns, and Jessica Sorrell.
\newblock Replicable reinforcement learning.
\newblock In Alice Oh, Tristan Naumann, Amir Globerson, Kate Saenko, Moritz
  Hardt, and Sergey Levine, editors, \emph{Advances in Neural Information
  Processing Systems 36: Annual Conference on Neural Information Processing
  Systems 2023, NeurIPS 2023, New Orleans, LA, USA, December 10 - 16, 2023},
  2023.

\bibitem[Esfandiari et~al.(2023{\natexlab{a}})Esfandiari, Kalavasis, Karbasi,
  Krause, Mirrokni, and Velegkas]{esfandiari2023replicableb}
Hossein Esfandiari, Alkis Kalavasis, Amin Karbasi, Andreas Krause, Vahab
  Mirrokni, and Grigoris Velegkas.
\newblock Replicable bandits.
\newblock In \emph{The Eleventh International Conference on Learning
  Representations}, 2023{\natexlab{a}}.

\bibitem[Esfandiari et~al.(2023{\natexlab{b}})Esfandiari, Karbasi, Mirrokni,
  Velegkas, and Zhou]{esfandiari2023replicable}
Hossein Esfandiari, Amin Karbasi, Vahab Mirrokni, Grigoris Velegkas, and Felix
  Zhou.
\newblock Replicable clustering.
\newblock In Alice Oh, Tristan Naumann, Amir Globerson, Kate Saenko, Moritz
  Hardt, and Sergey Levine, editors, \emph{Advances in Neural Information
  Processing Systems 36: Annual Conference on Neural Information Processing
  Systems 2023, NeurIPS 2023, New Orleans, LA, USA, December 10 - 16, 2023},
  2023{\natexlab{b}}.

\bibitem[Fotakis et~al.(2021)Fotakis, Kalavasis, Kontonis, and
  Tzamos]{fotakis2021efficient}
Dimitris Fotakis, Alkis Kalavasis, Vasilis Kontonis, and Christos Tzamos.
\newblock Efficient algorithms for learning from coarse labels.
\newblock In \emph{Conference on Learning Theory}, pages 2060--2079. PMLR,
  2021.

\bibitem[Freksen(2021)]{freksen2021introduction}
Casper~Benjamin Freksen.
\newblock An introduction to johnson-lindenstrauss transforms.
\newblock \emph{arXiv preprint arXiv:2103.00564}, 2021.

\bibitem[Freund and Schapire(1997)]{freund1997decision}
Yoav Freund and Robert~E Schapire.
\newblock A decision-theoretic generalization of on-line learning and an
  application to boosting.
\newblock \emph{Journal of computer and system sciences}, 55\penalty0
  (1):\penalty0 119--139, 1997.

\bibitem[Freund and Schapire(1998)]{freund1998large}
Yoav Freund and Robert~E Schapire.
\newblock Large margin classification using the perceptron algorithm.
\newblock In \emph{Proceedings of the eleventh annual conference on
  Computational learning theory}, pages 209--217, 1998.

\bibitem[Ghazi et~al.(2021)Ghazi, Kumar, and Manurangsi]{ghazi2021user}
Badih Ghazi, Ravi Kumar, and Pasin Manurangsi.
\newblock User-level differentially private learning via correlated sampling.
\newblock In Marc'Aurelio Ranzato, Alina Beygelzimer, Yann~N. Dauphin, Percy
  Liang, and Jennifer~Wortman Vaughan, editors, \emph{Advances in Neural
  Information Processing Systems 34: Annual Conference on Neural Information
  Processing Systems 2021, NeurIPS 2021, December 6-14, 2021, virtual}, pages
  20172--20184, 2021.

\bibitem[Goel et~al.(2020)Goel, Gollakota, and Klivans]{goel2020statistical}
Surbhi Goel, Aravind Gollakota, and Adam Klivans.
\newblock Statistical-query lower bounds via functional gradients.
\newblock \emph{Advances in Neural Information Processing Systems},
  33:\penalty0 2147--2158, 2020.

\bibitem[Gr{\o}nlund et~al.(2020)Gr{\o}nlund, Kamma, and
  Larsen]{gronlund2020near}
Allan Gr{\o}nlund, Lior Kamma, and Kasper~Green Larsen.
\newblock Near-tight margin-based generalization bounds for support vector
  machines.
\newblock In \emph{International Conference on Machine Learning}, pages
  3779--3788. PMLR, 2020.

\bibitem[Gupta et~al.(2011)Gupta, Hardt, Roth, and Ullman]{gupta2011privately}
Anupam Gupta, Moritz Hardt, Aaron Roth, and Jonathan Ullman.
\newblock Privately releasing conjunctions and the statistical query barrier.
\newblock In \emph{Proceedings of the forty-third annual ACM symposium on
  Theory of computing}, pages 803--812, 2011.

\bibitem[Impagliazzo et~al.(2022)Impagliazzo, Lei, Pitassi, and
  Sorrell]{impagliazzo2022reproducibility}
Russell Impagliazzo, Rex Lei, Toniann Pitassi, and Jessica Sorrell.
\newblock Reproducibility in learning.
\newblock In Stefano Leonardi and Anupam Gupta, editors, \emph{{STOC} '22: 54th
  Annual {ACM} {SIGACT} Symposium on Theory of Computing, Rome, Italy, June 20
  - 24, 2022}, pages 818--831. {ACM}, 2022.
\newblock \doi{10.1145/3519935.3519973}.
\newblock URL \url{https://doi.org/10.1145/3519935.3519973}.

\bibitem[Indyk and Motwani(1998)]{indyk1998approximate}
Piotr Indyk and Rajeev Motwani.
\newblock Approximate nearest neighbors: towards removing the curse of
  dimensionality.
\newblock In \emph{Proceedings of the thirtieth annual ACM symposium on Theory
  of computing}, pages 604--613, 1998.

\bibitem[Johnson(1984)]{johnson1984extensions}
William~B Johnson.
\newblock Extensions of lipshitz mapping into hilbert space.
\newblock In \emph{Conference modern analysis and probability, 1984}, pages
  189--206, 1984.

\bibitem[Kalavasis et~al.(2023)Kalavasis, Karbasi, Moran, and
  Velegkas]{kalavasis2023statistical}
Alkis Kalavasis, Amin Karbasi, Shay Moran, and Grigoris Velegkas.
\newblock Statistical indistinguishability of learning algorithms.
\newblock In Andreas Krause, Emma Brunskill, Kyunghyun Cho, Barbara Engelhardt,
  Sivan Sabato, and Jonathan Scarlett, editors, \emph{International Conference
  on Machine Learning, {ICML} 2023, 23-29 July 2023, Honolulu, Hawaii, {USA}},
  volume 202 of \emph{Proceedings of Machine Learning Research}, pages
  15586--15622. {PMLR}, 2023.
\newblock URL \url{https://proceedings.mlr.press/v202/kalavasis23a.html}.

\bibitem[Kalavasis et~al.(2024)Kalavasis, Karbasi, Velegkas, and
  Zhou]{kalavasis2024computational}
Alkis Kalavasis, Amin Karbasi, Grigoris Velegkas, and Felix Zhou.
\newblock On the computational landscape of replicable learning.
\newblock \emph{arXiv preprint arXiv:2405.15599}, 2024.

\bibitem[Kaplan et~al.(2020)Kaplan, Mansour, Stemmer, and
  Tsfadia]{kaplan2020private}
Haim Kaplan, Yishay Mansour, Uri Stemmer, and Eliad Tsfadia.
\newblock Private learning of halfspaces: Simplifying the construction and
  reducing the sample complexity.
\newblock \emph{Advances in Neural Information Processing Systems},
  33:\penalty0 13976--13985, 2020.

\bibitem[Karbasi et~al.(2023)Karbasi, Velegkas, Yang, and
  Zhou]{karbasi2023replicability}
Amin Karbasi, Grigoris Velegkas, Lin Yang, and Felix Zhou.
\newblock Replicability in reinforcement learning.
\newblock In Alice Oh, Tristan Naumann, Amir Globerson, Kate Saenko, Moritz
  Hardt, and Sergey Levine, editors, \emph{Advances in Neural Information
  Processing Systems 36: Annual Conference on Neural Information Processing
  Systems 2023, NeurIPS 2023, New Orleans, LA, USA, December 10 - 16, 2023},
  2023.

\bibitem[Kindler et~al.(2012)Kindler, Rao, O'Donnell, and
  Wigderson]{kindler2012spherical}
Guy Kindler, Anup Rao, Ryan O'Donnell, and Avi Wigderson.
\newblock Spherical cubes: optimal foams from computational hardness
  amplification.
\newblock \emph{Communications of the ACM}, 55\penalty0 (10):\penalty0 90--97,
  2012.

\bibitem[Kohler and Lucchi(2017)]{kohler17subsampled}
Jonas~Moritz Kohler and Aurelien Lucchi.
\newblock Sub-sampled cubic regularization for non-convex optimization.
\newblock In Doina Precup and Yee~Whye Teh, editors, \emph{Proceedings of the
  34th International Conference on Machine Learning}, volume~70 of
  \emph{Proceedings of Machine Learning Research}, pages 1895--1904. PMLR,
  06--11 Aug 2017.
\newblock URL \url{https://proceedings.mlr.press/v70/kohler17a.html}.

\bibitem[L{\^e}~Nguyen et~al.(2020)L{\^e}~Nguyen, Ullman, and
  Zakynthinou]{le2020efficient}
Huy L{\^e}~Nguyen, Jonathan Ullman, and Lydia Zakynthinou.
\newblock Efficient private algorithms for learning large-margin halfspaces.
\newblock In \emph{Algorithmic Learning Theory}, pages 704--724. PMLR, 2020.

\bibitem[Malliaris and Moran(2022)]{malliaris2022unstable}
Maryanthe Malliaris and Shay Moran.
\newblock The unstable formula theorem revisited.
\newblock \emph{arXiv preprint arXiv:2212.05050}, 2022.

\bibitem[Moran et~al.(2023)Moran, Schefler, and Shafer]{moran2023bayesian}
Shay Moran, Hilla Schefler, and Jonathan Shafer.
\newblock The bayesian stability zoo.
\newblock In Alice Oh, Tristan Naumann, Amir Globerson, Kate Saenko, Moritz
  Hardt, and Sergey Levine, editors, \emph{Advances in Neural Information
  Processing Systems 36: Annual Conference on Neural Information Processing
  Systems 2023, NeurIPS 2023, New Orleans, LA, USA, December 10 - 16, 2023},
  2023.

\bibitem[Nissim et~al.(2016)Nissim, Stemmer, and Vadhan]{nissim2016locating}
Kobbi Nissim, Uri Stemmer, and Salil~P. Vadhan.
\newblock Locating a small cluster privately.
\newblock In Tova Milo and Wang{-}Chiew Tan, editors, \emph{Proceedings of the
  35th {ACM} {SIGMOD-SIGACT-SIGAI} Symposium on Principles of Database Systems,
  {PODS} 2016, San Francisco, CA, USA, June 26 - July 01, 2016}, pages
  413--427. {ACM}, 2016.
\newblock \doi{10.1145/2902251.2902296}.
\newblock URL \url{https://doi.org/10.1145/2902251.2902296}.

\bibitem[Pineau et~al.(2019)Pineau, Sinha, Fried, Ke, and
  Larochelle]{pineau2019iclr}
Joelle Pineau, Koustuv Sinha, Genevieve Fried, Rosemary~Nan Ke, and Hugo
  Larochelle.
\newblock Iclr reproducibility challenge 2019.
\newblock \emph{ReScience C}, 5\penalty0 (2):\penalty0 5, 2019.

\bibitem[Rosenblatt(1958)]{rosenblatt1958perceptron}
Frank Rosenblatt.
\newblock The perceptron: a probabilistic model for information storage and
  organization in the brain.
\newblock \emph{Psychological review}, 65\penalty0 (6):\penalty0 386, 1958.

\bibitem[Valiant(1984)]{valiant1984theory}
Leslie~G Valiant.
\newblock A theory of the learnable.
\newblock \emph{Communications of the ACM}, 27\penalty0 (11):\penalty0
  1134--1142, 1984.

\bibitem[Vapnik(1999)]{vapnik1999nature}
Vladimir Vapnik.
\newblock \emph{The nature of statistical learning theory}.
\newblock Springer science \& business media, 1999.

\bibitem[Vapnik(2006)]{vapnik2006estimation}
Vladimir Vapnik.
\newblock \emph{Estimation of dependences based on empirical data}.
\newblock Springer Science \& Business Media, 2006.

\end{thebibliography}

\appendix
\onecolumn

\ifarxiv
\else
\section{Comparison of Results}\label{apx:detailed work summary}
\begin{table}[!ht]
    \centering
    \bigtableCaption
    \label{tab:detailed work summary}
    \bigtable
\end{table}
\fi

\section{Deferred Tools}

\subsection{SVM guarantees}
The following result is a restatement of Theorem 2 from
\citet{gronlund2020near}

\begin{lemma}[SVM Generalization Guarantee \citep{gronlund2020near}]
\label{lemma:svm-margin}
Let $\calD$ be a distribution over $\reals^d \times \{-1,1\}.$
Let $n \in \nats, S = (x_1,y_1),\ldots,(x_n,y_n)\sim \calD^n$, and $w \in \reals^d$ be 
a unit vector such that $y_i\left(w^\top \frac{x_i}{\|x_i\|}\right) \geq \tau, \forall i \in [n]$ Then, for
every $\delta > 0$ with probability at least
$1-\delta$ over the random draw of $S$, it holds that
\[
    \Pr_{(x,y) \sim \calD}[y\cdot(w^\top x/\|x\|) < \tau/2] \leq O\left(\frac{\log n + \log(\nicefrac1\delta)}{\tau^2 n}\right) \,.
\]
\end{lemma}
We remark that even though the result of \citet{gronlund2020near} is
stated for the generalization error with respect to the misclassification probability, i.e., $\Pr_{(x,y) \sim \calD}[y\cdot(w^\top x/\|x\|) \leq 0]$,
their argument also applies to the $\tau/2$-margin loss, i.e.,
$\Pr_{(x,y) \sim \calD}[y\cdot(w^\top x/\|x\|) < \tau/2]$, 
via a straightforward modification of some constants. 
In more detail, the only needed change is in the proof of part 1 of Claim 10 in \citet{gronlund2020near}. 
Here the first condition says $y\cdot(x^\top w) \leq 0$
and could be made e.g. $y\cdot(x^\top w) \leq \tau/4$.
Their result then holds with $\tau/4$ instead of $\tau/2$.

\subsection{Vector-Valued Bernstein Concentration Inequality}
We will use the following concentration inequality for the norm of random vectors.
\begin{lemma}[Vector Bernstein; \citep{kohler17subsampled}]\label{lemma:bernstein average}
    Let $X_1, \dots, X_B$ be independent random vectors
    with common dimension $d$ satisfying the following for all $i\in [B]$:
    \begin{enumerate}[(i)]
        \item $\E[X_i] = 0$
        \item $\norm{X_i}\leq \mu$
        \item $\E[\norm{X_i}^2]\leq \sigma^2$
    \end{enumerate}
    Let $Z := \frac1B \sum_{i=1}^B X_i$.
    Then for any $t\in (0, \nicefrac{\sigma^2}{\mu})$,
    \[
        \Pr[\norm{Z} \geq t]
        \leq \exp\left( -\frac{t^2 B}{8\sigma^2} + \frac14 \right).
    \]
\end{lemma}

\subsection{Johnson-Lindenstrauss Lemma}
\label{proof:JL}
{The remarkable result of \citet{johnson1984extensions} states that an appropriately scaled random orthogonal projection matrix preserves the norm of a unit vector with high probability.
\citet{indyk1998approximate} showed that it suffices to independently sample each entry of the matrix from the standard normal distribution.
\citet{achlioptas2001database} further simplified the construction to independently sample each entry from the Rademacher distribution.
See \citet{freksen2021introduction} for a detailed survey of the development.}

From hereonforth,
we say that a random matrix $A\in \reals^{k\times d}$ is a \emph{JL-matrix}
if either $A_{ij} \sim_{i.i.d.} \calN(0, 1/k)$
or $A_{ij} \sim_{i.i.d.} U\{-1/\sqrt{k}, +1/\sqrt{k}\}$.

We first state the standard distributional formulation of JL.

\begin{lemma}
[Distributional JL; \citep{johnson1984extensions,indyk1998approximate,achlioptas2001database}]
\label{lemma:distr-JL}
    Fix $\eps, \delta_{JL}\in (0, 1)$.
    Let $A \in \reals^{k \times d}$ be a JL-matrix
    for $k = \Omega(\eps^{-2} \log(\nicefrac1{\delta_{JL}}))$.
    Then for any $x \in \reals^d$,
    \[
        {\Pr_A[ \abs*{\|Ax\|^2 - \|x\|^2} > \eps \|x\|^2] \leq \delta_{JL}\,.}
    \]
\end{lemma}

Let $T$ be a set of vectors.
{By applying \Cref{lemma:distr-JL} to the $O(|T|^2)$ vectors $u-v$ for $u, v\in T$
and taking a union bound,
we immediately deduce the following result.}
\begin{lemma}
[JL Projection; \citep{johnson1984extensions,indyk1998approximate,achlioptas2001database}]
\label{lemma:JL}
    {Fix $\eps, \delta_{JL}\in (0, 1)$.}
    Consider a set $T$ of $d$-dimensional vectors 
    and a JL-matrix $A \in \reals^{k \times d}$
    for {$k = \Omega(\eps^{-2}\log(\nicefrac{|T|}{\delta_{JL}}))$}.
    Then,
    \[
        {\Pr_A \left[ \exists ~ {u, v\in T} :   \abs*{\| A(u - v) \|^2 - \|u-v\|^2} > \eps \|u-v\|^2 \right] 
        \leq {\delta_{JL}}} \,.
    \]
\end{lemma}

An application of \Cref{lemma:JL} towards the polarization identity for $z, x\in \reals^d$
\[
    4z^\top x
    = \|z+x\|^2 - \|z-x\|^2
\]
yields the following inner product preservation guarantee.

\begin{corollary}[JL Inner Product Preservation]\label{cor:JL-inner-product-preservation}
    {Fix $\eps, {\delta_{JL}}\in (0, 1)$}.
    Let $A \in \reals^{k \times d}$ be a JL-matrix
    for {$k = \Omega(\eps^{-2}\log(\nicefrac1{\delta_{JL}}))$}.
    Then, for any $x, z\in \R^d$,
    \[
        \Pr_A[|z^\top x - (Az)^\top Ax| > \eps \|z\|\cdot \|x\|] \leq {\delta_{JL}}
        \,.
    \]
\end{corollary}

The next lemma is another simple implication of the distributional JL.
\begin{lemma}
[Lemma 5 in \citet{le2020efficient}]
\label{lemma:JL-imply}
    Let $\calD$ satisfy \Cref{def:margin} with margin $\tau$.
    For a JL-matrix $A$ as stated in \Cref{lemma:distr-JL}, 
    let $G_A \subseteq \reals^d \times \{-1,1\}$ be the set of points $(x,y)$ of the population that satisfy
    \begin{itemize}
        \item $\abs*{\|Ax\|^2 - \|x\|^2} \leq \tau \|x\|^2/100$ and
        \item 
        $y (A w^\star/\|A w^\star\|)^\top (A x/\|A x\|) \geq 96 \tau/100$.
    \end{itemize}
    Then it holds that {$\Pr_{A, (x, y)\sim \cal D}[(x,y) \in G_A] \geq 1-\delta_{JL}$}.
\end{lemma}

\subsubsection{The Proof of \Cref{lem:whp-good-choice-JL-mapping}}
\label{apx:whp-good-choice-JL-mapping}
We finally prove \Cref{lem:whp-good-choice-JL-mapping},
whose statement we repeat below for convenience.
\whpGoodChoiceJL*

\begin{proof}
[Proof of \Cref{lem:whp-good-choice-JL-mapping}]
    We know from \Cref{lemma:JL-imply} that 
    \begin{align*}
        \E_{A}\left[\Pr_{(x,y) \sim \calD}[(x,y) \notin G_A]\right] &\leq \delta_{JL} \,,
    \end{align*}
    so Markov's inequality gives that 
    \begin{align*}
        \Pr_{A} \left[\Pr_{(x,y) \sim \calD}[(x,y) \notin G_A] \geq \eps' \right]
        \leq \frac{\E_{A}\left[\Pr_{(x,y) \sim \calD}[(x,y) \notin G_A]\right]}{\eps'}
        \leq \frac{\delta_{JL}}{\eps'} \,.
    \end{align*}
    Similarly, the guarantees of the JL projection immediately yields 
    \[
        \Pr_A \left[ \abs*{\|Aw^\star\|^2 - \|w^\star\|^2} \leq \tau \|w^\star\|^2/100\right] \leq \delta_{JL}.
    \]
\end{proof}

\section{Details for \Cref{section:rounding}}\label{apx:rounding}
Here we fill in the missing proofs from \Cref{section:rounding}.

First,
we restate and prove \Cref{lemma:round}.
\roundingReplicable*

\begin{proof}
[Proof of \Cref{lemma:round}]
    Fix a coordinate $i \in [k]$.
    The probability that $o(z)[i] \neq o(z')[i]$ is at most $|z[i]-z'[i]| \beta^{-1}$.
    Assume $o(z)[i] = o(z')[i]$. Then, by the definition of $p(z), p(z')$ we have that 
    \begin{align*}
        |z[i]-z'[i]| 
        &= \left| (p(z)[i]-p(z')[i])o(z)[i] \right. \\
        &\qquad + \left. (p(z')[i]-p(z)[i])(o(z)[i]+\beta) \right| \\
        &= |(p(z')[i]-p(z)[i])\beta|. 
    \end{align*}
    Note then the probability of $f_{o, u}(z)[i] \neq f_{o, u}(z')[i]$ is $|p(z')[i]-p(z)[i]| = |z[i]-z'[i]| \beta^{-1}$.
    
    By the uniform choice of $u_i$, 
    we thus conclude that $\Pr_{o,u}[f_{o,u}(z)[i] \neq f_{o,u}(z')[i]] \leq 2\beta^{-1} \abs{z[i] - z'[i]}$. 
    A union bound over all $k$ coordinates implies $\Pr_{o,u}[f_{o,u}(z) \neq f_{o,u}(z')] \leq 2\beta^{-1} \|z-z'\|_1$.
\end{proof}

Next,
we prove \Cref{lemma:inner},
whose statement is repeated.
\roundingInnerProduct*

\begin{proof}
[Proof of \Cref{lemma:inner}]
    Each of the random variables $f_{o,u}(z)[i]$ lies in an interval of length $\beta$,
    are independent,
    and have expectation $z[i]$. 
    By linearity,
    $f_{o,u}(z)^\top x -  z^\top x  = (f_{o,u}(z)-z)^\top x$. 
    Let $v = f_{o,u}(z)-z$. 
    Then $\E[x[i]\cdot v[i]]=0$ and $x[i]\cdot v[i]$ lies in a range of length $\beta x[i]$. 
    By Hoeffding's inequality, 
    we have $\Pr[| v^\top x | > \alpha] < 2 \exp(-2\alpha^2/(\sum_i \beta^2x[i]^2)) \leq 2 \exp(-2 \alpha^2 \beta^{-2})$.
\end{proof}

\section{Details for \Cref{alg:algo2}}\label{apx:algo2}
Here we fill in the missing proofs from \Cref{sec:algo2}.

We begin with \Cref{lem:projecting rounding preserves margin},
whose statement we repeat below for convenience.
\projectingRoundingMargin*
\begin{proof}[Proof of \Cref{lem:projecting rounding preserves margin}]
    Let $\calD_z := \calD \mid \set{y(z^\top x/\norm{x}) \geq \tau/2}$ be the distribution obtained from $\calD$ by conditioning on $z$ having large margin
    and let $\calD_b$ denote the distribution $\calD$ conditioned on the complement.
    We can decompose $\calD = (1-p_b)\cdot \calD_z + p_b\cdot \calD_b$
    for $p_b\leq \eps$.
    To complete the
    correctness argument, we need to show
    that with high probability,
    the projection step of $z, x$ onto
    the low-dimensional space $Az, Ax$ and the rounding $b = \AKround(Az, \beta)$
    approximately preserves the inner product $y(b^\top Ax/\norm{x})$
    for a $1-\eps$ fraction of $\calD_z$.
    Then 
    the final classifier $A^\top b$ still has low population error.
    In particular,
    it suffices to show that $\Pr_{(x, y)\sim \calD_z}[y(b^\top Ax/\norm{x}) < \tau/4]\leq \eps$ with probability at least $1-\delta$ over the
    random choice of $A, b$.
    Then, the total population error is at most $\eps + p_b\leq 2\eps$ with probability at least $1-2\delta$.
    
    Fix $(x, y)\in \supp(\calD_z)$ and remark that
    \begin{align*}
    &y\cdot b^\top \frac{Ax}{\norm{x}} \\
    &= y\cdot (Az)^\top A\left( \frac{x}{\norm{x}} \right) - y\cdot (Ax - b)^\top \frac{Ax}{\norm{Ax}}.
    \end{align*}
    By the choices of
    \begin{align*}
    k \geq {\Omega}(\tau^{-2}\log(\nicefrac1{\eps\delta})), \qquad
    \beta \leq O(\tau/\log(\nicefrac1{\eps\delta})), 
    \end{align*}
    the JL lemma (cf. \Cref{cor:JL-inner-product-preservation})
    and Alon-Kartag rounding scheme (cf. \Cref{lemma:inner}) ensures that
    \[
    \E_{A, b} \left[ \Pr_{(x, y)\sim \calD_z} [y(b^\top Ax/\norm{x}) < \tau/4] \right] \leq \eps\delta.
    \]
    An application of Markov's inequality yields 
    \[
    \Pr_{A, b} \left[ \Pr_{(x, y)\sim \calD_z} [y(b^\top Ax/\norm{x}) < \tau/4] >\eps \right] \leq \delta.
    \]
    All in all,
    with probability at least $1-\delta$ over $A, b$, the population error is at most
    \begin{align*}
    &\Pr_{(x,y)\sim\calD}[y(A^\top b)^\top x \leq 0] \\
    &\leq \Pr_{(x,y)\sim\calD} [y(b^\top Ax/\norm{x}) < \Theta(\tau)]
    < 2\eps,    
    \end{align*}
    concluding
    the correctness argument of our algorithm.
\end{proof}

Next,
we prove \Cref{lem:projecting rounding is replicable},
whose statement we again repeat.
\projectingRoundingReplicable*
\begin{proof}[Proof of \Cref{lem:projecting rounding is replicable}]
    An application of the vector Bernstein concentration inequality (cf. \Cref{lemma:bernstein average}) yields the following:
    let $S_1 = \left\{w_i^{(1)}\right\}_{i \in [B]}$ and $S_2 = \left\{w_i^{(2)}\right\}_{i \in [B]}$ be two sets of independent and identically distributed random vectors in $\reals^d$ with $\left\|w_i^{(j)}\right\| \leq 1$ for all $i\in [B], j \in \{1,2\}$.
    {Then we can set $X_i = w_i^{(1)} - \E w$.
    Notice that $\norm{X_i}\leq 2$ and $\norm{X_i}^2\leq 4$.
    Hence, an application of \Cref{lemma:bernstein average} yields for any $t\in (0, 2)$
    \begin{align*}
        &\Pr_{S_1} \left[ \left\| \frac{1}{B} \sum w_i^{(1)} - \E[w] \right\| \geq t \right] \\
        &\leq 
        \exp\left( -\frac{t^2 B}{32} + \frac14 \right)
        = \exp(\Omega(-t^2 B))
    \end{align*}
    Let $z = \frac{1}{B} \sum_{i \in [B]} w_i^{(1)}, z' = \frac{1}{B} \sum_{i \in [B]} w_i^{(2)}.$
    Choosing $B\geq \Omega(t^{-2}\log(\nicefrac1\rho))$ ensures that
    $
        \Pr_{S_1, S_2}
        \left[
        \left\|
        z
        -
        z'
        \right\|
        \geq t
        \right]
        \leq 
        \rho/10.
    $
    
    We condition on the event that
    $\left\|z
    -
    z'
    \right\|
    \leq t.$
    Since we are sharing the randomness
    across the two executions, 
    we use the same JL projection matrix $A$.
    Thus our choice of $k\geq \Omega(\tau^{-2} \log(\nicefrac1\rho))$ 
    gives that
    $
        \left\|
        Az
        -
        Az'\right\|
        \leq (1+\tau)t\leq 2t,
    $
    with probability at least {$1-\rho/10$} (cf. \Cref{lemma:JL}).
    It remains to show that
    after the rounding step, with probability
    at least $1-\rho/10$, the two rounded vectors
    will be the same.
    The size of our rounding grid is $\beta = \tilde\Theta(\tau)$ and the target dimension of our JL-matrix is $k = \tilde\Theta(\tau^{-2})$. 
    Thus by \Cref{lemma:round},
    the probability that the two points round to different vectors is 
    $\Theta(\beta^{-1}) \|Az-Az'\|_1 
    \leq \Theta(\beta^{-1}\sqrt{k}) \|Az-Az'\|_2  
    = \Theta(t\sqrt{k} / \beta)$ (cf. \Cref{lemma:round}). Thus, if we pick
    \begin{align}
        t &\leq \rho\beta/\sqrt{k} \nonumber \\
        B 
        &\geq \tilde\Omega(k/(\beta^2\rho^2))
        = {\Omega}(\log(\nicefrac1{\eps\tau\rho\delta}) / (\tau^{4} \rho^{2})),
    \end{align}}
    we complete the argument. 
\end{proof}

\section{Details for \Cref{alg:algo4}}\label{apx:algo4}
Here we fill in the missing proofs from \Cref{sec:algo4}.

We now prove \Cref{lem:boosted SGD},
whose statement we repeat below for convenience.
\boostedSGD*
\begin{proof}[Proof of \Cref{lem:boosted SGD}]
    Let $T = \Omega(\tau^{-2}\eps^{-2})$.
    Run $\SGD(\calD, T)$ for $n = O(\log(\nicefrac1\delta))$ times
    to obtain solutions $w_1, \dots, w_n$.
    \Cref{thm:sgd convergence} ensures that we attain an $\eps/10$-optimal solution in expectation for each $w_i, i\in [n]$.
    Markov's inequality then guarantees that we attain an $\eps/5$-optimal solution with probability at least $1/2$ in each repetition.
    But then with probability at least $1-\delta/10$,
    at least one of the $n$ solutions is $\eps/5$-optimal.

    By a Hoeffding bound,
    we can estimate each $f_\calD(w_i)\in [0, O(1/\tau)]$ up to an additive $\eps/10$ error with probability at least $1-\delta/10$ using $O(n\eps^{-2}\tau^{-2}\log(\nicefrac{n}\delta))$ samples.
    Outputting the classifier with the lowest estimated objective yields a $3\eps/10$-optimal solution
    with probability at least $1-\delta/5$.

    Normalizing the selected classifier to a unit vector can only decrease $h$
    since the feasible region is $B_1^d$
    and incurs an additional loss of $\eps/10$ due to the regularizer.
    Thus we have a $4\eps/10$-optimal solution with probability at least $1-\delta/5$.

    The total sample complexity is
    $
        O(nT) + O(n\eps^{-2}\tau^{-2}\log(\nicefrac{n}\delta))
        = \tilde O(\eps^{-2}\tau^{-2}\log(\nicefrac1\delta)).
    $
\end{proof}

\subsection{Proof of \Cref{thm:main4}}
\label{proof:algo4}

\paragraph{Correctness of \Cref{alg:algo4}.}
From the choice of
$
    n \geq \Omega(\eps^{-2} \tau^{-2} \log(\nicefrac{B}\delta)),
$
\Cref{lem:boosted SGD} ensures that each unit vector $w_i$ produced in \Cref{alg:algo4} is $\eps/10$-optimal with probability at least $1-\delta/(10B)$.
Hence with probability at least $1-\delta/10$, every $w_i$ is $\eps/10$-optimal with respect to $f_\calD$.
By Jensen's inequality,
$z = (1/B) \sum_{i=1}^B w_i$ is $\eps/10$-optimal with probability at least $1-\delta/10$.
But then by the choice of $f_\calD$ as a convex surrogate loss for $\ones\set{y(x^\top w) < \tau/2}$,
$z$ satisfies $\Pr_{(x, y)\sim \calD} [y(z^\top x/\norm{x}) < \tau/2] \leq 2\eps/10$ with probability at least $1-\delta/10$.
In other words,
the unit vector $z$ has a population $\tau/2$-loss of at most $2\eps/10$ with probability at least $1-\delta/10$.
But then similar to the correctness of \Cref{alg:algo2},
an application of \Cref{lem:projecting rounding preserves margin} concludes the proof of correctness.

\paragraph{Replicability of \Cref{alg:algo4}.}
Similar to the correctness of \Cref{alg:algo2},
we note that the output $w_i$ of each execution of $\boostSGD$ is an i.i.d. unit vector.
Thus an application of \Cref{lem:projecting rounding is replicable} yields the replicability guarantees of \Cref{alg:algo4}.

\paragraph{Sample Complexity \& Running Time of \Cref{alg:algo4}.}
The sample complexity is $nB = \tilde O(\eps^{-2}\tau^{-6}\rho^{-2}\log(\nicefrac1\delta))$ as required.
Once again,
we see that \Cref{alg:algo4} terminates in $\poly(d, 1/\eps, 1/\tau, 1/\rho, \log(\nicefrac1\delta))$ by inspection.

\section{Details for \Cref{prop:dp-reduction-sample-complexity}}

\subsection{Differential Privacy}
 For $a,b, \alpha,\beta \in [0,1]$, let $a \approx_{\alpha, \beta} b$ denote the statement $a \leq e^\alpha b + \beta$ and $b \leq e^\alpha a + \beta$. We say that two probability distributions $P, Q$ are $(\alpha, \beta)$-indistinguishable if $P(E) \approx_{\alpha, \beta} Q(E)$ for any measurable event $E$.\footnote{We use the notation $(\alpha,\beta)-$DP instead of the more common $(\eps,\delta)-$DP to be 
 consistent with the notation of the rest of the paper
 regarding the accuracy and probability of failure
 of the learning algorithms.} 
\begin{definition}
[Approximate Differential Privacy]
\label{def:dp}
A learning rule $A$
is an $n$-sample $(\alpha,\beta)$-differentially private if for any pair of datasets $S, S' \in (\calX \times \{0,1\})^n$ that differ on a single example, the induced posterior distributions $A(S)$ and $A(S')$ are $(\alpha, \beta)$-indistinguishable.
\end{definition}

\subsection{The Results of \citet{le2020efficient} and \citet{bun2023stability}}
\label{sec:prior-results}
\Cref{prop:dp-reduction-sample-complexity} is based on combining the following two results. 
The first is a differentially private learner for large-margin halfspaces from \citet{le2020efficient}.
\begin{proposition}
[Theorem 6 in \citet{le2020efficient}]
\label{prop:lydia}
Let $\alpha, \tau, \eps, \delta > 0$.
Let $\calD$ be a distribution over $\reals^d \times \{-1,1\}$ that has linear margin $\tau$ as in \Cref{def:margin}. 
There is an algorithm that is $(\alpha,0)$-differentially private and, given $m = \wt{O}(\alpha^{-1} \eps^{-1} \tau^{-2})$ i.i.d. samples $(x,y) \sim \calD$, computes in time $\exp(1/\tau^2)\poly(d, 1/\alpha, 1/\eps, \log(\nicefrac1\delta))$ a normal vector $w \in \reals^d$ such that
$\Pr_{(x,y) \sim \calD}[\sgn(w^\top x) \neq y] \leq \eps$, with probability at least $1-
\delta$.
\end{proposition}

We also use the DP-to-Replicability reduction appearing in \citet{bun2023stability}.
\begin{proposition}
[Corollary 3.18 in \citet{bun2023stability}]
\label{prop:reduction}
Fix $n \in \mathbb N$, sufficiently small $\rho \in (0,1)$,
$\eps, \delta\in (0, 1)$
and $\alpha, \beta > 0$.
Let $\calA : \calX^n \to \calY$ be an $n$-sample $(\alpha,\beta)$-differentially private algorithm with \textbf{finite output space}
solving a statistical task with accuracy $\eps$ and failure probability $\delta$.
Then there is an algorithm $\calA' : \calX^{m} \to \calY$
that is $\rho$-replicable and solves the same statistical task with $m = O(\rho^{-2} n^2)$ samples with accuracy $\eps$ and failure probability $\delta$.
\end{proposition}

\end{document}